\documentclass[journal]{IEEEtran}
\usepackage[utf8]{inputenc}

\usepackage{amsfonts}
\usepackage{amsmath}
\usepackage{amsthm}
\usepackage{cancel}
\usepackage{mathtools}
\usepackage{mathrsfs}
\usepackage{pifont}

\newcommand{\transpose}{^{\top}}

\newcommand{\greekcolvec}[1]{\boldsymbol{#1}}

\newtheorem{theorem}{Theorem}
\newtheorem{lemma}{Lemma}
\newtheorem{proposition}{Proposition}

\theoremstyle{definition}
\newtheorem{definition}{Definition}

\theoremstyle{remark}
\newtheorem*{remark}{Remark}

\newcommand{\ie}{i.e., }
\newcommand{\eg}{e.g., }
\newcommand{\vs}{vs.\ }
\newcommand{\etal}{et al.}
\newcommand{\Poincare}{Poincar{\'e}\ }

\usepackage[nolist,nohyperlinks]{acronym}
\acrodef{BOA}[BOA]{basin of attraction }
\acrodef{COM}[COM]{center of mass }
\acrodef{DEM}[DEM]{discrete element method }
\acrodef{DOF}[DOF]{degree-of-freedom }
\acrodef{FSM}[FSM]{finite state machine }
\acrodef{DOF}[DOF]{degree of freedom }
\acrodef{GM}[GM]{granular media }
\acrodef{GRF}[GRF]{ground reaction force }
\acrodef{HZD}[HZD]{hybrid zero dynamics }
\acrodef{IMU}[IMU]{inertial measurement unit}
\acrodef{LCP}[LCP]{linear complementarity problem}
\acrodef{LO}[LO]{liftoff}
\acrodef{MRAC}[MRAC]{model reference adaptive control }
\acrodef{MPC}[MPC]{model predictive control}
\acrodef{NLP}[NLP]{nonlinear program}
\acrodef{ODE}[ODE]{ordinary differential equation }
\acrodef{PMP}[PMP]{Pontryagin's Maximum Principle }
\acrodef{RFT}[RFT]{resistive force theory }
\acrodef{RL}[RL]{reinforcement learning }
\acrodef{STC}[STC]{self-tuning control}
\acrodef{SLIP}[SLIP]{spring-loaded inverted pendulum }
\acrodef{SQP}[SQP]{sequential quadratic programming}
\acrodef{TD}[TD]{touchdown}
\acrodef{TPBVP}[TPBVP]{two-point boundary value problem}
\acrodef{ZMP}[ZMP]{zero-moment point }

\usepackage{graphicx}
\usepackage{caption}
\usepackage{subcaption}
\usepackage{epstopdf}

\usepackage[normalem]{ulem}
\usepackage{xcolor}   
\newcommand{\noop}[1]{}
\newcommand{\rev}[1]{{\color{black} #1}}  
\newcommand{\strout}[1]{\noop{#1}}        

\usepackage[numbers,sort&compress]{natbib}
\usepackage{soul}

\usepackage{url}
\usepackage{hyperref}
\hypersetup{
	colorlinks=true,
	linkcolor=blue,
	filecolor=magenta,      
	urlcolor=blue,
}

\title{Efficient, Responsive, and Robust Hopping on Deformable Terrain}
\author{Daniel J. Lynch, Jason L. Pusey, Sean W. Gart, Paul B. Umbanhowar, and Kevin M. Lynch

\thanks{Research was sponsored by the Army Research Laboratory and was accomplished under Cooperative Agreement Number W911NF-24-2-0032. The views and conclusions contained in this document are those of the authors and should not be interpreted as representing the official policies, either expressed or implied, of the Army Research Laboratory or the U.S. Government. The U.S. Government is authorized to reproduce and distribute reprints for Government purposes notwithstanding any copyright notation herein.
\textit{(Corresponding author: Daniel J. Lynch.)}}
\thanks{D. J. Lynch, P. B. Umbanhowar, and K. M. Lynch are with the Center for Robotics and Biosystems and the Department of Mechanical Engineering, Northwestern University, Evanston, IL 60208 USA (email: daniellynch2021@u.northwestern.edu; umbanhowar@northwestern.edu; kmlynch@northwestern.edu).
K. M. Lynch is also with the Northwestern Institute of Complex Systems, Evanston, IL 60201 USA.

J. L. Pusey and S. W. Gart are with DEVCOM Army Research Lab, 2800 Powder Mill Road, Adelphi, MD 20783, USA (email: jason.l.pusey.civ@army.mil; sean.w.gart.civ@army.mil).
}
} 

\begin{document}
\bstctlcite{IEEEexample:BSTcontrol} 

\maketitle

\begin{abstract}
Legged robot locomotion is hindered by a mismatch between applications where legs can outperform wheels or treads, most of which feature deformable substrates, and existing tools for planning and control, most of which assume flat, rigid substrates.
In this study we focus on the ramifications of plastic terrain deformation on the hop-to-hop energy dynamics of a spring-legged monopedal hopping robot animated by a switched-compliance energy injection controller.
From this deliberately simple robot-terrain \rev{template}, we derive a hop-to-hop energy return map, and we use physical experiments and simulations to validate the hop-to-hop energy map for a real robot hopping on a real deformable substrate.
The dynamical properties (fixed points, eigenvalues, basins of attraction) of this map provide insights into efficient, responsive, and robust locomotion on deformable terrain.
Specifically, we identify constant-fixed-point surfaces in a controller parameter space that suggest it is possible to tune control parameters for efficiency or responsiveness while targeting a desired gait energy level.
We also identify conditions under which fixed points of the energy map are globally stable, and we further characterize the basins of attraction of fixed points when these conditions are not satisfied.
We conclude by discussing the implications of this hop-to-hop energy map for planning, control, and estimation for efficient, agile, and robust legged locomotion on deformable terrain.
\end{abstract}


\section{Introduction}\label{sec:introduction}
\IEEEPARstart{L}{egged} robots have the potential to outperform wheeled or tracked robots on uneven and deformable terrain, which comprises most of the Earth's landmass and many extraterrestrial bodies~\cite{bagnold1974physics,perry2022Bennu}, yet state-of-the-art tools for planning and controlling legged locomotion generally assume flat and rigid ground \cite{Westervelt_2003_HZD} or rely on a preponderance of data and extensive \textit{a priori} machine learning to enable off-road forays \cite{choi2023learning}.
The applicability of legged robots to major challenges of the 21st century (\eg humanitarian aid, disaster response, terrestrial and extraterrestrial exploration, sustainable agriculture, environmental restoration, deep sea mining)~\cite{guizzo2015DRC,krotkov2017DRC} motivates the development of legged robots capable of traversing uneven and deformable substrates.
Legged locomotion on these substrates is characterized by irrecoverable energy loss through plastic terrain deformation, complex coupling between ground reaction forces and foot kinematics, and spatiotemporal terrain heterogeneity~\cite{Li_Terradynamics_2013}.
While natural substrates like sand, soil, snow, gravel, and leaf litter are often complex and heterogeneous, nature abounds with legged animals that traverse these substrates with ease, including both invertebrates (\eg arachnids, crustaceans, insects) and vertebrates (\eg amphibians, birds, lizards, mammals) \cite{alexander2003principles}.
Although the complexity and variety of natural substrates defies description with a unified model such as a terrestrial analog of the Navier-Stokes equations \cite{jaeger1996granular,askariIntrusion_NM2016}, the fact that legged locomotion has evolved across such a wide variety of animals and environments suggests the existence of underlying principles that can be applied to advance the state of the art in legged robotic locomotion.

In this work, we study the hop-to-hop energy dynamics of a simple vertically-constrained monopedal hopping robot and how it is affected by irrecoverable energy loss through plastic terrain deformation underfoot.
We focus on vertically-constrained hopping on a single leg \rev{not because we view it as a practical means of traversing deformable terrain but, rather,} because it is the simplest setting to study the interplay between foot-ground interaction, plastic terrain deformation, and hop-to-hop energy dynamics, \rev{and because it has the potential to inform design and control choices that enable practical legged robots to traverse deformable terrain}.
\strout{Specifically}\rev{With these goals in mind}, we develop a one-dimensional map that maps the robot's kinetic energy from one hop to the next, given the work done by the robot on itself and on the ground between hops, parameterized by dimensionless quantities derived from a terrain model, a robot model, and a robot controller.
We then characterize the dynamical properties of this map (fixed points, eigenvalues, basins of attraction) in terms of control, design, and terrain parameters and relate these dynamical properties to properties of legged robot locomotion, namely efficiency, agility, and robustness.
We conclude by discussing the implications of these relationships for planning, control, and terrain estimation.

\subsection{Background}
To traverse any substrate, rigid or deformable, a legged robot or animal must replace energy lost during each step by doing work,~\ie injecting energy.
For any legged locomotor that interacts with the ground one foot at a time, this \strout{step-to-step} energy balance is
\rev{\begin{equation}\label{eq:generic_energy_map}
    E_\mathrm{TD}\left[k+1\right] = E_\mathrm{TD}\left[k\right] + E_\mathrm{inj} - E_\mathrm{loss},
\end{equation}
where $E_\mathrm{TD}\left[k\right]$ is the total mechanical energy of the locomotor at the $k$-th touchdown (\ie when the foot makes contact with the ground), $E_\mathrm{TD}\left[k+1\right]$ is the total mechanical energy of the locomotor at the next touchdown,} and $E_\mathrm{inj}$ and $E_\mathrm{loss}$ are the energies injected and lost, respectively, during the stance phase \rev{beginning at the $k$-th touchdown}.
In general, $E_\mathrm{loss}$ and $E_\mathrm{inj}$ are functions of \rev{$E_\mathrm{TD}[k]$}, robot control and design parameters, and terramechanical parameters.
On deformable terrain, our focus, energy is lost via plastic ground deformation underfoot, which occurs whenever the robot pushes on the ground with a force greater than the current yield threshold~\cite{jackson1983granular}, which, for homogeneous deformable substrates (\eg sand, soil, snow), increases monotonically with foot penetration depth~\cite{jaeger1996granular,Li_Terradynamics_2013}.
Thus, energy loss and energy injection are coupled through the stance-phase foot kinematics.

Consider the case when the foot is at rest on the ground when the robot begins to extend its leg in order to jump:
if the robot pushes against the ground with a force less than the depth-dependent yield threshold, its foot remains stationary and the robot does work exclusively on itself;
conversely, if the robot pushes against the ground with a force greater than the depth-dependent yield threshold, the ground underfoot yields and the robot does work on the ground as well as on itself, thus increasing the energetic cost of locomotion.
In this sense, deformable terrain is distinct from rigid ground, where the yield threshold is effectively infinite.
Deformable terrain is also distinct from liquids and powders with near-zero yield stress~\cite{lohse2004quicksand,perry2022Bennu},
on which the locomotor sinks unless it can exploit effects such as surface tension, as is the case for water striders \cite{hu2003strider} and fishing spiders \cite{suter1999dolomedes}, or inertial drag, as is the case for basilisk lizards \cite{glasheen1996basilisk,hsieh2004running}.

\subsection{Related Work}
As hybrid nonlinear dynamical systems, legged locomotors are inherently complex.
Accordingly, the last half-century has seen the growth of a wide variety of increasingly sophisticated planning and control techniques attempting to endow legged robots with the agility, versatility, and efficiency of their animal counterparts \cite{vukobratovic1972ZMP,raibert1986legged,mcgeer1990passive,collins2005passive,Westervelt_2003_HZD,hereid2018dircolHZD}.
Early quasistatic methods based on the \ac{ZMP} enable statically-stable flat-footed bipedal walking \cite{vukobratovic1972ZMP} but are inherently ill-equipped to handle gaits that contain a flight phase, \eg hopping and running, or other underactuated locomotion phases.
In these cases, a different notion of stability---limit cycle stability---is required.

Limit cycle gaits correspond to fixed points of a \Poincare map and, if a fixed point exists, the local stability of the limit cycle is approximated to first order by the linearization of the \Poincare map about the fixed point \cite{guckenheimer2013nonlinear,goswami1996limit,full2002quantifying}.
Pioneering work by McGeer demonstrated how such a limit cycle gait could arise from the passive (\ie unactuated) dynamics of a bipedal robot walking down a gentle slope \cite{mcgeer1990passive}, and later Collins et al.\ showed how actuation could be introduced to achieve similar bipedal limit-cycle gaits on level ground \cite{collins2005passive}.
Due to their reliance on passive dynamics, these limit-cycle walkers are energetically efficient but are rarely capable of walking at more than one speed.
Subsequently, Westervelt \etal~developed \ac{HZD} as a constructive framework that ensures limit cycle existence and stability through a set of virtual constraints and feedback gains \cite{Westervelt_2003_HZD}.
While gaits generated by these techniques are locally dynamically stable by construction, they rely on restrictive assumptions about foot-ground contact (\eg idealizing the foot as a stationary body about which the leg pivots) and are often sensitive to model uncertainty~\cite{grizzle2014review}.

\rev{\subsubsection{Templates and Anchors}
In contrast to the branch of planning and control techniques that address the full robot dynamics (\eg those based on the \ac{ZMP} or \ac{HZD}), Raibert pioneered a different approach, stemming from a set of decoupled single-\ac{DOF} controllers that, when combined in parallel, enable asymptotically stable limit cycle hopping on rigid ground, whether hopping in place, forward at prescribed speeds, or along prescribed paths~\cite{raibert1986legged}.
In particular, Raibert's hop-height controller ensures cyclic vertical motion---a characteristic of running gaits in animals---by regulating the elastic potential energy injected into the leg spring during stance, which is converted into kinetic energy and gravitational potential energy during flight~\cite{raibert1986legged}.
Raibert's work overlapped with biomechanics research \cite{cavagna1977energy,mcmahon1984mechanics,mcmahon1985compliance,alexander1990three,mcmahon1990stiffness} that led to the development of the \ac{SLIP}, a low-dimensional energy-conserving dynamic model \cite{blickhan1989slip,blickhan1993similarity,geyer2005spring-mass,geyer2006compliant}.
The \ac{SLIP} is an example of a \textit{template}:
\begin{definition}[Template]
    Full and Koditschek define a template as the simplest model with the least number of variables and parameters that exhibits a behavior of interest~\cite{full1999templates}.
\end{definition}
While the \ac{SLIP} template embodies multiple simplifications (\eg a prismatic leg with negligible mass and a point foot that acts as a revolute joint during stance), it is nevertheless \textit{descriptive}, reproducing the \ac{COM} kinematics and \ac{GRF} profiles of a variety of bipedal and quadrupedal animals running across rigid ground~\cite{mcmahon1985compliance,blickhan1989slip,blickhan1993similarity}.
The \ac{SLIP} template has also emerged as a low-dimensional model informing the design and control of legged robots; in this sense, it is also \textit{prescriptive}~\cite{full1999templates,klavins2002decentralized,holmes2006dynamics,poulakakis2009ASLIP,hubicki2016atrias}.
While some robots have been built with prismatic legs~\cite{bares1989Ambler,bares1999DanteII}, most legged robots are more biomimetic in form and feature legs consisting of multiple links connected by revolute joints.
If the inertia of the leg links is small relative to the inertia of the body, the SLIP template can be embedded by using a suitable Jacobian transformation to convert forces commanded by the template to joint torques; if not, a more sophisticated approach is necessary.
An example of the latter is Poulakakis and Grizzle's use of \ac{HZD} to embed the dynamics of a modified \ac{SLIP} template in the robot leg ``Thumper''~\cite{poulakakis2009ASLIP}.
In contrast, Hubicki~\etal~designed the robot ``ATRIAS'' such that the \ac{SLIP} template is embedded primarily through the robot's passive mechanical properties rather than through control~\cite{hubicki2016atrias}.

While some recent work has taken a computationally-intensive optimization-based approach to achieving dynamic legged robot locomotion~\cite{erez2013humanoid_RT_MPC,dicarlo2018MIT_Cheetah_3}, template-based control still has the potential to offer a number of benefits such as interpretability, composability, and tractability.
By ``interpretability'' we mean not only the ability to certify that a given controller can achieve the task set before it, but also the ability to explain how the controller achieves this task and how the controller is affected by changes in the model or the task.
Template-based control is, to a large extent, an application of nonlinear dynamical systems theory and therefore template-based controllers admit analysis in terms of limit cycles, hybrid invariant attractors, averaging theory, and other formal techniques, many of which can be used to certify controller performance \cite{full1999templates,klavins2002decentralized,Westervelt_2003_HZD,holmes2006dynamics,poulakakis2009ASLIP,de2018averaging}.
This gives template-based control an advantage, in terms of transparency, over control methods based on numerically solving high-dimensional nonlinear constrained optimization problems~\cite{neunert2018NMPC}.

The second desirable characteristic of template-based control is composability. A multi-degree-of-freedom behavior like monopedal running naturally emerges from the composition of several lower-dimensional constituents, e.g., cyclic vertical and angular motion and steady horizontal motion~\cite{raibert1986legged}.
A second layer of composition, then, enables polypedal running: each leg functions as an oscillator and running is the steady-state response when these oscillators are coupled together with appropriate gains and phase differences \cite{raibert1986quad,saranli2001rhex,klavins2002decentralized,de2018quad}.
Composability connotes reusability and generalizability: templates drawn from a sufficiently rich library can be recombined to form new behaviors, and the same control templates can be anchored to a variety of robot platforms.
Hoeller~\etal~demonstrate the potential of this composition-based approach in a learning-based framework that enables quadrupedal parkour-like multimodal locomotion by drawing from a library of several locomotion primitives learned in simulation~\cite{hoeller2024anymal}.

Tractability, the third desirable characteristic of template-based control frameworks, is a consequence of building these frameworks around low-dimensional templates such that less onboard computation is required compared to optimization-based control.
When computational resources are not allocated to basic tasks like locomotion, they are available to serve higher-level objectives like perception, navigation, and human interaction.
}

\subsubsection{Hopping Templates}
\rev{Raibert's monopedal and, later, quadrupedal robots served as a major inspiration for template-based control, in part because they demonstrated impressive robustness despite a number of modeling simplifications, and also in part because of their low bandwidth requirements for sensing and control computation~\cite{raibert1984experiments,raibert1986quad}.}
Koditschek and B{\"u}hler offer a justification, rooted in dynamical systems theory, for \strout{the success of Raibert's hoppers} \rev{this robust behavior} by showing that \rev{Raibert's} hop-height controller results in an energy return map with an essentially globally stable fixed point, due to its topological conjugacy to a unimodal map whose Schwarzian derivative is negative everywhere \cite{koditschek1991hopping,singer1978unimodal,guckenheimer1979unimodal}.
Koditschek and B{\"u}hler~\cite{koditschek1991hopping} and Vakakis \etal~\cite{vakakis1991strange} both showed that Raibert's hop-height controller could also produce more complex behavior, including higher-period gaits and a period-doubling route to chaotic hopping.
\rev{M'Closkey and Burdick then showed that this complex behavior persists when hopping with nonzero forward velocity~\cite{mcloskey1993forward}, suggesting that dynamical results obtained for vertical hopping retain relevance for running gaits.}
\rev{The analytical approach embodied in these papers}, rooted in energy return map analysis, continues to bear fruit: for example, more recently, Degani~\etal~used a similar approach to study dynamic climbing via jumping between two walls of an inclined chute \cite{degani2014parkourbot}.

\subsubsection{Legged Robotic Locomotion on Deformable Terrain}
In contrast to level rigid ground, natural terrain exhibits spatiotemporal heterogeneity.
It is unsurprising, therefore, that the study of legged robot locomotion on natural terrain is at a significantly earlier developmental stage than its rigid-ground counterpart, despite the greater need for legged robots capable of traversing natural terrain \cite{krotkov2017DRC}.
\strout{Indeed, the assumptions underlying Raibert's hoppers, the \ac{SLIP} template, and \ac{HZD}}
\rev{Indeed, the assumptions underlying Raibert's hoppers and the \ac{SLIP} template }
(\ie point feet, flat and rigid ground) break down on naturally-occurring deformable substrates, where finite foot area is necessary and where foot-ground contact cannot be idealized as a revolute joint.

The need for controllable and repeatable terrain conditions to support model-based controller design for legged robots on natural terrain motivates the use of granular media---collections of discrete particles that interact primarily through friction and repulsion \cite{jaeger1996granular}---as a proxy for naturally-occurring deformable substrates \cite{qian2015principles,jin2019preparation}.
Li \etal~showed that \ac{RFT} can be used to model foot-ground interaction on granular media and that, to first order, the \ac{GRF} increases linearly with foot penetration depth.
\rev{Li~\etal~also demonstrated that, for monodisperse substrates consisting of roughly spherical particles, the mapping from intruder kinematics to local stress is the same up to a single substrate-specific scaling factor; in other words, results obtained with poppy seed substrates (for example) naturally generalize to other dry granular substrates~\cite{Li_Terradynamics_2013}.
We note that further complications like interstitial wetting and polydispersity are interesting and relevant to field applications but are beyond the scope of our present work.}
The development of granular \ac{RFT} has produced many offshoots: for example, Xiong \etal~apply granular \ac{RFT} to quasistatic bipedal locomotion on granular media by devising constraints on the projection of the \ac{COM} onto the stance-phase support polygon that integrate with controllers based on either the \ac{ZMP} or \ac{HZD} \cite{xiong2017stability};
meanwhile, other granular media locomotion studies examine more dynamic behavior, including jumping from rest \cite{Aguilar_2016_Robophysical,Hubicki_Tractable_2016}, minimizing energy loss on impact \cite{lynch2020SLP}, and repeated jumping on yielding substrates with a modified Raibert-like hop height controller \cite{roberts2018reactive,roberts2019mitigating}.

In this paper, we use tools from nonlinear dynamics and one-dimensional maps to unify these different aspects of hopping on granular media, addressing the effect of plastic ground deformation underfoot on the hop-to-hop energy dynamics of a spring-legged monopedal robot---\rev{a vertical hopping template}---animated by a Raibert-like energy injection controller.
In particular, given Koditschek and B{\"u}hler's justification for the success of Raibert's hop-height controller on hard ground \cite{koditschek1991hopping}, it is natural to ask whether a similar controller will enable similarly robust hopping on deformable terrain.
The answer, as we show in this paper, is not a simple yes-or-no but depends on a combination of control and terrain parameters.
Furthermore, we show that these parameters affect not only the steady-state response and its basin of attraction but also the transient response, resulting in tradeoffs between efficiency, agility, and robustness.
Together, these parameters form a low-dimensional space that can be exploited for gait control and planning as well as for online terrain estimation.

\subsection{Statement of Contributions}
\rev{
In this paper, we make three contributions to the state of the art in template-based control for running on deformable terrain.
\begin{enumerate}
    \item We derive a hop-to-hop energy return map from a hybrid dynamics model of hopping on deformable terrain, which we view as a necessary component of template-based running on soft ground. The hybrid dynamics model includes a two-phase model of foot-ground interaction that explicitly accounts for a finite, depth-dependent yield threshold and a \ac{GRF} that increases monotonically with depth, which we believe to be the features of deformable terrain most relevant to legged locomotion. We validate this model through extensive experiments with a vertically constrained monopod hopping on granular media. To facilitate generalization of our work to robots of different length and mass scales and to deformable substrates of different stiffnesses, we formulate the hopping model and the resulting hop-to-hop energy map in terms of four dimensionless variables.
    \item We derive a closed-form expression for the hop-to-hop energy map in the limit of negligible foot mass. We prove that the map has a unique fixed point in this limit. We then derive closed-form expressions for the fixed point and eigenvalue in terms of dimensionless model parameters which enable our results to generalize to legged robots of different sizes and to yielding substrates of different stiffnesses. We show through extensive simulations that the dynamical properties of the map are relatively unaffected by the ratio of foot mass to body mass (\ie the ratio of unsprung mass to sprung mass) once this ratio is smaller than approximately 1:10, which agrees with prior work concerning hopping and running on rigid ground.
    \item We propose metrics for efficiency and an energy stability margin, based on the fixed point and eigenvalue of the energy return map, and, in the massless-foot limit, we identify necessary and sufficient conditions for global stability of the map's fixed point. We then characterize the effect of the dimensionless model parameters on these properties of the template, revealing a range of possible transient responses and basins of attraction for a given steady-state response. Finally, we discuss how this result sheds light on tradeoffs between efficiency, agility, and robustness to terrain variability and has implications for the design of planners, controllers, and estimators for legged locomotion on yielding terrain.
\end{enumerate}
}

\strout{We derive a one-dimensional return map that captures the hop-to-hop energy dynamics of a vertically-constrained spring-legged monopod hopping on plastically deformable terrain, which we model as a unidirectional spring, \ie a spring that resists compression but not tension, resulting in a yield threshold that increases monotonically with foot penetration depth.
This hop-to-hop energy map is parameterized by four dimensionless quantities: the energy injected normalized by the characteristic energy associated with the weight of the robot and the ground stiffness; the ratio of leg stiffness to ground stiffness; the ratio of foot mass to body mass; and the ratio of force applied upon energy injection to the maximum force the ground can support without yielding, given the foot penetration depth when energy is injected.
When this force ratio exceeds one, the ground reyields underfoot on push-off, dissipating some of the energy that would otherwise be injected and qualitatively changing the hop-to-hop energy dynamics.
We validate our robot-terrain model and hop-to-hop energy map with simulations and physical experiments.}

\strout{The hop-to-hop energy map provides a theoretical framework for discussing efficiency, agility, and robustness---all relevant to legged locomotion on deformable terrain---in terms of the map's fixed points, eigenvalues, and basins of attraction.
Specifically, we propose definitions of efficiency and a stability margin (a proxy for agility), and we identify necessary and sufficient conditions for which fixed points of the hop-to-hop energy map are globally stable.
Moreover, in the limit of negligible foot mass, the hop-to-hop energy map has a closed-form representation, enabling analysis of the relationship between dimensionless parameters and hop-to-hop energy dynamics.
This analysis reveals a range of possible transient responses and basins of attraction for a given steady-state response, a result which sheds light on tradeoffs between efficiency, agility, and robustness to terrain variability and has implications for the design of planners, controllers, and estimators for legged locomotion on yielding terrain.}

\section{Modeling}\label{sec:modeling}
In this section, we develop a model of a monopedal robot hopping on plastically deformable terrain, shown in Figure \ref{fig:model}.
\strout{While this model is intentionally simple, we show in Section \ref{sec:experiments} that it qualitatively describes the dynamics of a real robot hopping on a real deformable substrate.
We also show in Section \ref{sec:simulations} that it is capable of producing surprisingly complex behavior including chaotic hopping.
We then analyze the dynamics associated with this model in greater detail in Section \ref{sec:dynamics_and_massless_foot_analysis}.}
\rev{While this model is intentionally simple, it qualitatively describes the dynamics of a real monopod hopping on a real deformable substrate and can produce complex behavior including chaotic hopping.}

\subsection{Terrain Model}\label{sec:terrain_model}
While nature abounds with deformable substrates, their outdoor existence generally precludes their use in repeatable and controllable robotic locomotion studies.
However, granular media---collections of discrete particles that interact primarily through repulsion and friction~\cite{jaeger1996granular}---are useful as tunable proxies for naturally-occurring deformable substrates and are used in numerous robotic locomotion studies~\cite{li2009sensitive,Li_Terradynamics_2013,Aguilar_2016_Robophysical,Hubicki_Tractable_2016,roberts2018reactive,roberts2019mitigating,lynch2020SLP}.
Here, we use the terrain model from our previous study of minimum-penetration impact into granular media~\cite{lynch2020SLP}, in which we model the granular substrate under the robot foot as a unidirectional spring,~\ie a spring which resists compression with a force proportional to intruder depth but which does not spring back.
While this model neglects substrate-specific and foot-geometry-specific velocity-dependent effects (\eg inertial drag~\cite{katsuragi2007unified} and granular accretion~\cite{Aguilar_2016_Robophysical}) as well as elasticity in the grains themselves, it captures three important characteristics of intrusion into dry granular media.
First, ground reaction forces increase monotonically with intruder depth.
Second, tensile stresses are much weaker than compressive stresses.
Third, and most relevant to our present work, a finite yield stress represents a threshold for applied stress, above which the material transitions from a solid-like state to a fluid-like state~\cite{jaeger1996granular,askariIntrusion_NM2016}.
Specifically, our model for the \ac{GRF} on the robot foot is
\begin{align}\label{eq:GRF}
f_\mathrm{g} =
\begin{cases}
0,&\hspace{-0.4cm}\text{~if~}q_\mathrm{f} \geq 0\text{~or~}\dot{q}_\mathrm{f} > 0\text{~(flight),}\\
-k_\mathrm{g} q_\mathrm{f},&\hspace{-0.4cm}\text{~if~}q_\mathrm{f} < 0\text{~and~}\dot{q}_\mathrm{f} < 0\text{~(yielding stance),}\\
f_\mathrm{static},&\hspace{-0.4cm}\text{~if~}q_\mathrm{f} < 0\text{~and~}\dot{q}_\mathrm{f} = 0\text{~(static stance),}
\end{cases}
\end{align}
where $f_\mathrm{g}$ is a function of the foot position $q_\mathrm{f}$ (defined positive upward, relative to undeformed terrain), the foot velocity $\dot{q}_\mathrm{f}$, and the ground stiffness $k_\mathrm{g}$ (which is proportional to both the depth-dependent yield stress per unit depth~\cite{Li_Terradynamics_2013} and the foot area).
In ``yielding stance'' the foot is intruding and the ground is deforming underfoot; in ``static stance'' the foot is at rest and the ground is not deforming underfoot but rather is applying a positive normal force $f_\mathrm{static} \in \left[0,-k_\mathrm{g}q_\mathrm{f}\right]$.

Our focus on vertical hopping in this paper stems from \rev{the desire} to understand the ramifications of a finite yield threshold on the hop-to-hop energy dynamics, as vertical hopping is the simplest setting in which to study this interaction.
As such, we follow Roberts and Koditschek in assuming the ground is undeformed at the beginning of each hop, thereby omitting the added complications of hop-to-hop ground compaction associated with overlapping, and therefore interacting, footsteps from our terrain model \cite{roberts2019mitigating}.

\subsection{Robot Model}\label{sec:robot_model}
\begin{figure}[t]
    \centering
    \includegraphics[width=0.9\linewidth]{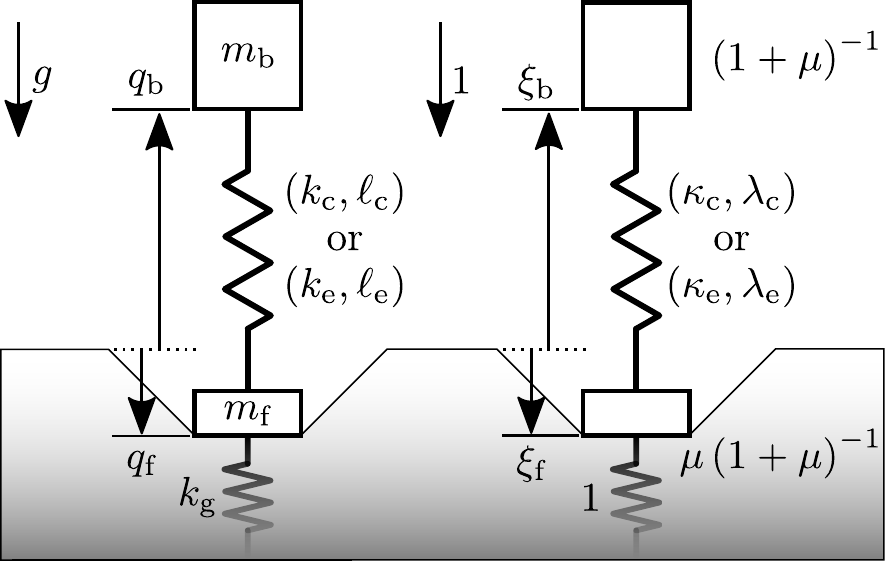}
    \caption{Terrain and robot model, shown with dimensional variables (left) and dimensionless variables (right). See Sections \ref{sec:terrain_model} and \ref{sec:robot_model} for details on terrain model and robot model, respectively, and Table \ref{tab:dim_to_nondim} for dimensionless nomenclature.
    }
    \label{fig:model}
\end{figure}
\rev{Because our focus in this paper is on the cyclic vertical motion associated with running gaits, we model the robot} as a vertically-constrained monopod with body mass $m_\mathrm{b}$ and foot mass $m_\mathrm{f}$.
The position of the body and the foot relative to undeformed terrain are $q_\mathrm{b}$ and $q_\mathrm{f}$, respectively.
The body and foot are coupled by an actuator that emulates a linear spring.
We assume the actuator is an ideal force source with perfect efficiency, and we neglect actuator stroke and force limits\footnote{The return map analysis described in this paper can be adapted to accommodate actuator inefficiency and stroke and force limits. Given the variety of actuators used in legged robots, however, attempting to incorporate an actuator model adds complexity without shedding further light on the ramifications of ground deformation on energy dynamics of legged locomotion on soft soil.}.
To counteract the energy lost to plastic terrain deformation underfoot, a Raibert-style stance-phase controller injects energy by updating the stiffness and unloaded length of the emulated leg spring when the leg is maximally compressed \cite{raibert1986legged}.
A flight-phase controller resets the leg length and brings the foot to rest relative to the body prior to the next \ac{TD}.

To ensure the results in this paper generalize to legged locomotion across a variety of scales, we nondimensionalize the model using characteristic mass $m_\mathrm{u} = m_\mathrm{b} + m_\mathrm{f}$, length $q_\mathrm{u} = m_\mathrm{u}g/k_\mathrm{g}$, and time $t_\mathrm{u} = \sqrt{m_\mathrm{u}/k_\mathrm{g}}$ from the harmonic oscillator formed by the total monopod mass and the ground spring.
We introduce a mass ratio $\mu$, defined as $\mu = m_\mathrm{f}/m_\mathrm{b}$, such that the limit case of negligible foot mass is represented by $\mu=0$.
\rev{All dimensionless model variables are summarized in Table~\ref{tab:dim_to_nondim}.}
The resulting nondimensionalized model is
\begin{subequations}\label{eq:nondim_EOM}
	\begin{align}
	\ddot{\xi}_\mathrm{b} &= -1 + \left(1 + \mu\right)\phi_\mathrm{a},\\
	\ddot{\xi}_\mathrm{f} &= -1 + \left(\frac{1 + \mu}{\mu}\right)\left(\phi_\mathrm{g} - \phi_\mathrm{a}\right),
	\end{align}
\end{subequations}
where $\xi_\mathrm{b}$ and $\xi_\mathrm{f}$ are the positions of the body and foot, respectively; \rev{$\,\dot{}\,$ and $\,\ddot{}\,$ are first and second derivatives with respect to nondimensional time, respectively;} $\mu$ is the ratio of foot mass to body mass; $\phi_\mathrm{a}$ is the nondimensionalized force applied by the leg actuator; and $\phi_\mathrm{g}$ is the nondimensionalized \ac{GRF}.
The nondimensionalized force applied by the leg actuator is
\begin{equation}\label{eq:nondim_actuator_force}
    \phi_\mathrm{a} =
	\begin{cases}
		\kappa_\mathrm{c}\left(\lambda_\mathrm{c} - \xi_\mathrm{b} + \xi_\mathrm{f}\right)\textrm{, if }\dot{\xi}_\mathrm{b} \leq \dot{\xi}_\mathrm{f}\textrm{ (compression)},\\
				\kappa_\mathrm{e}\left(\lambda_\mathrm{e} - \xi_\mathrm{b} + \xi_\mathrm{f}\right)\textrm{, if }\dot{\xi}_\mathrm{b} > \dot{\xi}_\mathrm{f}\textrm{ (extension)},
	\end{cases}
\end{equation}
where $\kappa_\mathrm{c}$ and $\lambda_\mathrm{c}$ are the compression-mode stiffness and unloaded length, respectively, of the emulated leg spring and $\kappa_\mathrm{e}$ and $\lambda_\mathrm{e}$ are the extension-mode stiffness and unloaded length, respectively, of the emulated leg spring.
The leg switches from compression to extension at time $\tau_\mathrm{CE}$ and extends continuously until liftoff, when the foot separates from the ground.
The nondimensionalized \ac{GRF} is
\begin{align}\label{eq:nondim_grf}
    \phi_\mathrm{g} =
	\begin{cases}
		0,&\textrm{if }\xi_\mathrm{f} > 0\textrm{ or }\dot{\xi}_\mathrm{f} > 0\textrm{ (flight)},\\ 
		-\xi_\mathrm{f},&\textrm{if }\xi_\mathrm{f} \leq 0\textrm{ and }\dot{\xi}_\mathrm{f} < 0\textrm{ (yielding stance)},\\ 
		\phi_\mathrm{a} + \frac{\mu}{1 + \mu},&\textrm{if }\xi_\mathrm{f} < 0\textrm{ and }\dot{\xi}_\mathrm{f} = 0\textrm{ (static stance)}.
	\end{cases}
\end{align}
\strout{All dimensionless variables in Equations \eqref{eq:nondim_EOM}, \eqref{eq:nondim_actuator_force}, and \eqref{eq:nondim_grf} are summarized in Table~\ref{tab:dim_to_nondim}.}

 \begin{table}[t]
	\centering
	\begin{tabular}{r|l}
		Dimensionless & Physical \\
		expression      & quantity \\
		\hline
		$\tau$ & time\\
		$\xi_\mathrm{b}$ & body position\\
		$\xi_\mathrm{f}$ & foot position\\
        $\mu$ & ratio of foot mass to body mass\\
        $(1 + \mu)^{-1}$ & body mass\\
        $\mu(1 + \mu)^{-1}$ & foot mass\\
		$\phi_\mathrm{a}$ & force applied by leg spring\\
		$\phi_\mathrm{g}$ & ground reaction force (GRF)\\
        $\phi$ & force ratio: $-\left(\phi_\mathrm{a}\left(\tau_\mathrm{CE}\right) + \mu\left(1 + \mu\right)^{-1}\right)\xi_\mathrm{f}^{-1}\left(\tau_\mathrm{CE}\right)$\\
		$\kappa_\mathrm{c}$ & compression-mode leg stiffness\\
        $\kappa_\mathrm{e}$ & extension-mode leg stiffness\\
		$\lambda_\mathrm{c}$ & compression-mode unloaded leg length\\
        $\lambda_\mathrm{e}$ & extension-mode unloaded leg length\\
		$\varepsilon_\mathrm{TD}$ & \ac{COM} kinetic energy on \ac{TD}\\
        $\varepsilon_\mathrm{inj}$ & energy injected during a hop\\
        $\varepsilon_\mathrm{loss}$ & energy lost during a hop
	\end{tabular}
	\caption{Nomenclature. Dimensionless body and foot velocities (accelerations) are represented as $\dot{\xi}_\mathrm{b,f} =\mathrm{d}\xi_\mathrm{b,f}/\mathrm{d}\tau$ ($\ddot{\xi}_\mathrm{b,f} =\mathrm{d}^2\xi_\mathrm{b,f}/\mathrm{d}\tau^2$).}
	\label{tab:dim_to_nondim}
\end{table}

\subsection{Controllers}\label{sec:controller}
\subsubsection{Stance-Phase Controller}
When the body stops moving relative to the foot during stance, an instantaneous update in the monopod's leg spring stiffness from $\kappa_\mathrm{c}$ to $\kappa_\mathrm{e}$ and unloaded length $\lambda_\mathrm{c}$ to $\lambda_\mathrm{e}$ injects energy and changes the force applied to the ground by the foot.
The nondimensionalized energy $\varepsilon_\mathrm{inj}$ injected by this update is
\begin{align}\label{eq:injected_energy_dimensional}
\begin{split}
    \varepsilon_\mathrm{inj} ={}& \frac{1}{2}k_\mathrm{e}\left(\lambda_\mathrm{e} - \xi_\mathrm{b}\left(\tau_\mathrm{CE}\right) + \xi_\mathrm{f}\left(\tau_\mathrm{CE}\right)\right)^2 \\
    &-\frac{1}{2}\kappa_\mathrm{c}\left(\lambda_\mathrm{c} - \xi_\mathrm{b}\left(\tau_\mathrm{CE}\right) + \xi_\mathrm{f}\left(\tau_\mathrm{CE}\right)\right)^2,
\end{split}
\end{align}
where $\kappa_\mathrm{c}$ and $\lambda_\mathrm{c}$ are the compression-mode leg spring stiffness and unloaded length, respectively, $\kappa_\mathrm{e}$ and $\lambda_\mathrm{e}$ are the extension-mode leg spring stiffness and unloaded length, respectively, and $\xi_\mathrm{b}\left(\tau_\mathrm{CE}\right)$ and $\xi_\mathrm{f}\left(\tau_\mathrm{CE}\right)$ are the body and foot position, respectively, at the compression-extension transition.

The change in leg spring stiffness and unloaded length enables the monopod to replace energy lost to plastic ground deformation, but it also changes the force applied to the ground.
When the ratio of the total applied force to the depth-dependent yield threshold exceeds one, the ground underfoot reyields, resulting in energy dissipation.
The force ratio $\phi$ is 
\begin{equation}\label{eq:nondim_force_ratio}
    \phi = \frac{\kappa_\mathrm{e}\left(\lambda_\mathrm{e} - \xi_\mathrm{b}\left(\tau_\mathrm{CE}\right) + \xi_\mathrm{f}\left(\tau_\mathrm{CE}\right)\right) + \mu\left(1 + \mu\right)^{-1}}{-\xi_\mathrm{f}\left(\tau_\mathrm{CE}\right)}.
\end{equation}
\begin{remark}
    For the robot to jump, it must accelerate its body upward by applying a force greater than the body weight, \ie $\phi > (1 + \mu)^{-1}\xi_\mathrm{f}^{-1}(\tau_\mathrm{CE})$.
\end{remark}
While hopping without reyielding is energetically efficient \strout{(as we show in Sections \ref{sec:simulations} and \ref{sec:massless_foot_analysis})}, real-world deformable terrain features variability in ground stiffness and elevation, so reyielding is likely to occur at least occasionally.
Moreover, other practical robot-specific limitations (\eg limited leg stroke) may result in scenarios where reyielding is unavoidable for the locomotion task at hand.
Finally, \strout{as we show in Section \ref{sec:massless_foot_analysis},} it may be possible to exploit reyielding to converge to a particular gait in fewer hops than possible without reyielding.

Given compression-mode spring parameters $(\kappa_\mathrm{c},\lambda_\mathrm{c})$, a one-to-one mapping exists between extension-mode spring parameters $(\kappa_\mathrm{e},\lambda_\mathrm{e})$ and injected energy and force ratio $\left(\varepsilon_\mathrm{inj},\phi\right)$.
In keeping with our focus on the energetic and dynamical consequences of depth-dependent yield thresholds, we choose $\varepsilon_\mathrm{inj}$ and $\phi$ as our controls, and we solve for $\kappa_\mathrm{e}$ and $\lambda_\mathrm{e}$ on each hop using the following equations:
\begin{subequations}\label{eq:nondim_leg_spring_update}
	\begin{align}
		\kappa_\mathrm{e} =& \frac{\left(\phi \xi_\mathrm{f}\left(\tau_\mathrm{CE}\right) + \frac{\mu}{1 + \mu}\right)^2}{2\varepsilon_\mathrm{inj} + \kappa_\mathrm{c}\left(\lambda_\mathrm{c} - \xi_\mathrm{b}\left(\tau_\mathrm{CE}\right) + \xi_\mathrm{f}\left(\tau_\mathrm{CE}\right)\right)^2},\\
        \begin{split}
		\lambda_\mathrm{e} =& \frac{-2\varepsilon_\mathrm{inj} + \kappa_\mathrm{c}\left(\lambda_\mathrm{c} - \xi_\mathrm{b}\left(\tau_\mathrm{CE}\right) + \xi_\mathrm{f}\left(\tau_\mathrm{CE}\right)\right)^2}{\left(\phi \xi_\mathrm{f}\left(\tau_\mathrm{CE}\right) + \frac{\mu}{1 + \mu}\right)} \\
        &+ \xi_\mathrm{b}\left(\tau_\mathrm{CE}\right) - \xi_\mathrm{f}\left(\tau_\mathrm{CE}\right).
        \end{split}
	\end{align}
\end{subequations}
\begin{remark}
    Our stance-phase controller contrasts with Roberts and Koditschek's controller \cite{roberts2018reactive,roberts2019mitigating}.
    They fix $\lambda_\mathrm{e}$ equal to $\lambda_\mathrm{c}$ and prescribe an increase in $\kappa_\mathrm{e}$ relative to $\kappa_\mathrm{c}$; $\varepsilon_\mathrm{inj}$ and $\phi$ vary from hop to hop, although reyielding upon energy injection is inevitable on homogeneous terrain because $\phi > 1$ on every hop.
    In contrast, we prescribe $\varepsilon_\mathrm{inj}$ and $\phi$, and compute $\kappa_\mathrm{e}$ and $\lambda_\mathrm{e}$ online when the switching condition is triggered.
    While this approach requires sensing the \ac{GRF}, it provides an additional parameter to explore locomotion behavior with and without reyielding on deformable terrain.
\end{remark}

\subsubsection{Flight-Phase Controller}
The monopod enters flight at \ac{LO}, when the foot velocity becomes positive.
In preparation for the next stance phase, the flight-phase controller resets the leg length to $\lambda_\mathrm{c}$ and brings the foot to rest relative to the body.
For a nonzero foot mass, this reset dissipates energy:
\begin{align}\label{eq:nondim_energy_loss_at_liftoff}
    \varepsilon_\mathrm{loss,LO} = \mu\left(\frac{\mu + \kappa_\mathrm{e}\left(\dot{\xi}_\mathrm{b}(\tau_\mathrm{LO}) - \dot{\xi}_\mathrm{f}(\tau_\mathrm{LO})\right)^2}{2 \kappa_\mathrm{e}\left(1 + \mu\right)^2}\right),
\end{align}
where $\tau_\mathrm{LO}$ is the time at which \ac{LO} occurs.
Equation~\eqref{eq:nondim_energy_loss_at_liftoff} reflects an energetic benefit of lightweight feet: $\varepsilon_\mathrm{loss,LO}\to 0$ as $\mu\to 0$.

The compression-mode unloaded leg length $\lambda_\mathrm{c}$ determines the gravitational potential energy of the monopod on \ac{TD} but does not change $\varepsilon_\mathrm{TD}$, the \ac{COM} (center of mass) kinetic energy at \ac{TD}, given by
\begin{equation}\label{eq:nondim_CoM_KE_TD}
    \varepsilon_\mathrm{TD} = \frac{1}{2}\left(\frac{\rev{\dot{\xi}_\mathrm{b}\left(\tau_\mathrm{TD}\right) + \mu \dot{\xi}_\mathrm{f}\left(\tau_\mathrm{TD}\right)}}{1 + \mu}\right)^2,
\end{equation}
where $\tau_\mathrm{TD}$ is the time at which the next \ac{TD} occurs.
Thus, the nondimensionalized kinetic energy dynamics depend on four parameters, which we group into a set $\theta$ defined as
\begin{equation}\label{eq:nondim_params_set}
    \theta = \left\{\varepsilon_\mathrm{inj},\phi,\kappa_\mathrm{c},\mu\right\}.
\end{equation}

\subsection{Hybrid System Model}\label{sec:hybrid_model}
The robot-terrain model is an autonomous hybrid system, consisting of a set of domains (each defined by a contact mode and a control mode that, together, give rise to a vector field that determines the continuous dynamics on that domain), 
\rev{a set of switching conditions defined at the boundaries between domains, and a set of reset maps that update the state of the hybrid system upon transition to a new domain~\cite{hereid2018dircolHZD}.}
\strout{a set of switching surfaces defined at the boundaries between domains, and a set of jump maps defined on each switching surface that update the state of the hybrid system upon transition to a new domain \cite{hereid2018dircolHZD}.}
\rev{We denote the state of the monopod on each domain by 
\begin{equation}\label{eq:robot_state_vector}
    \greekcolvec{\xi} = [\xi_\mathrm{b},\xi_\mathrm{f},\dot{\xi}_\mathrm{b},\dot{\xi}_\mathrm{f}]\transpose.
\end{equation}
This hybrid system, represented in Figure \ref{fig:digraph} as a directed graph, consists of the following domains: flight (F), yielding stance with leg compression (YC), static stance with leg compression (SC), yielding stance with leg extension (YE), and static stance with leg extension (SE).}
The continuous dynamics on each stance domain are given in nondimensionalized form by Equations \eqref{eq:nondim_EOM}, \eqref{eq:nondim_actuator_force}, and \eqref{eq:nondim_grf}.
The \strout{switching surfaces} \rev{transition events} of this hybrid system are \textit{touchdown} (TD), in which the foot position crosses zero with negative velocity;
\textit{foot stop} (FS), in which the foot velocity crosses zero and the total force applied to the ground is less than the depth-dependent yield threshold;
\textit{reyield} (RY), in which the force applied to the ground first crosses the depth-dependent yield threshold;
\textit{compression-extension} (CE), in which the body velocity relative to the foot velocity crosses zero from below,~\ie the leg transitions from compression to extension; and
\textit{liftoff} (LO), in which the foot velocity becomes positive.
See Figure~\ref{fig:digraph} for mathematical definitions of switching conditions in terms of dimensionless quantities.
Note that for any force ratio $\phi > 1$, CE and RY transitions occur simultaneously as the force applied at the onset of extension exceeds the depth-dependent yield threshold, resulting in reyielding and energy dissipation\footnote{CE and FS transitions can also coincide, provided $\phi \leq 1$.}.

All \strout{jump} \rev{reset} maps in this hybrid system are identity maps except for the liftoff map \strout{$\Delta_\mathrm{LO}$}, which models the actions of the flight-phase controller as a perfectly inelastic collision that brings the foot to rest relative to the body and resets the leg to its compression-mode unloaded length.
\rev{The \ac{LO} reset map is
\begin{equation}\label{eq:LO_reset_map}
    \greekcolvec{\xi}^{\mathrm{LO}^+}
    =
    \begin{bmatrix}
        \frac{\mu}{1 + \mu} & \frac{1}{1 + \mu} & 0 & 0\\
        \frac{\mu}{1 + \mu} & \frac{1}{1 + \mu} & 0 & 0\\
        0 & 0 & \frac{\mu}{1 + \mu} & \frac{1}{1 + \mu}\\
        0 & 0 & \frac{\mu}{1 + \mu} & \frac{1}{1 + \mu}\\
    \end{bmatrix}
    \greekcolvec{\xi}^{\mathrm{LO}^-}
    +
    \begin{bmatrix}
        \frac{1}{1 + \mu}\\
        \frac{\mu}{1 + \mu}\\
        0\\
        0
    \end{bmatrix}
    \lambda_\mathrm{c},
\end{equation}
where the superscripts $\mathrm{LO}^-$ and $\mathrm{LO}^+$ are shorthand for $\tau_\mathrm{LO}^-$ and $\tau_\mathrm{LO}^+$, respectively, and indicate the pre- and post-liftoff values of the state vector $\greekcolvec{\xi}$.}

\begin{figure}[t]
    \centering
    \includegraphics[width=\linewidth]{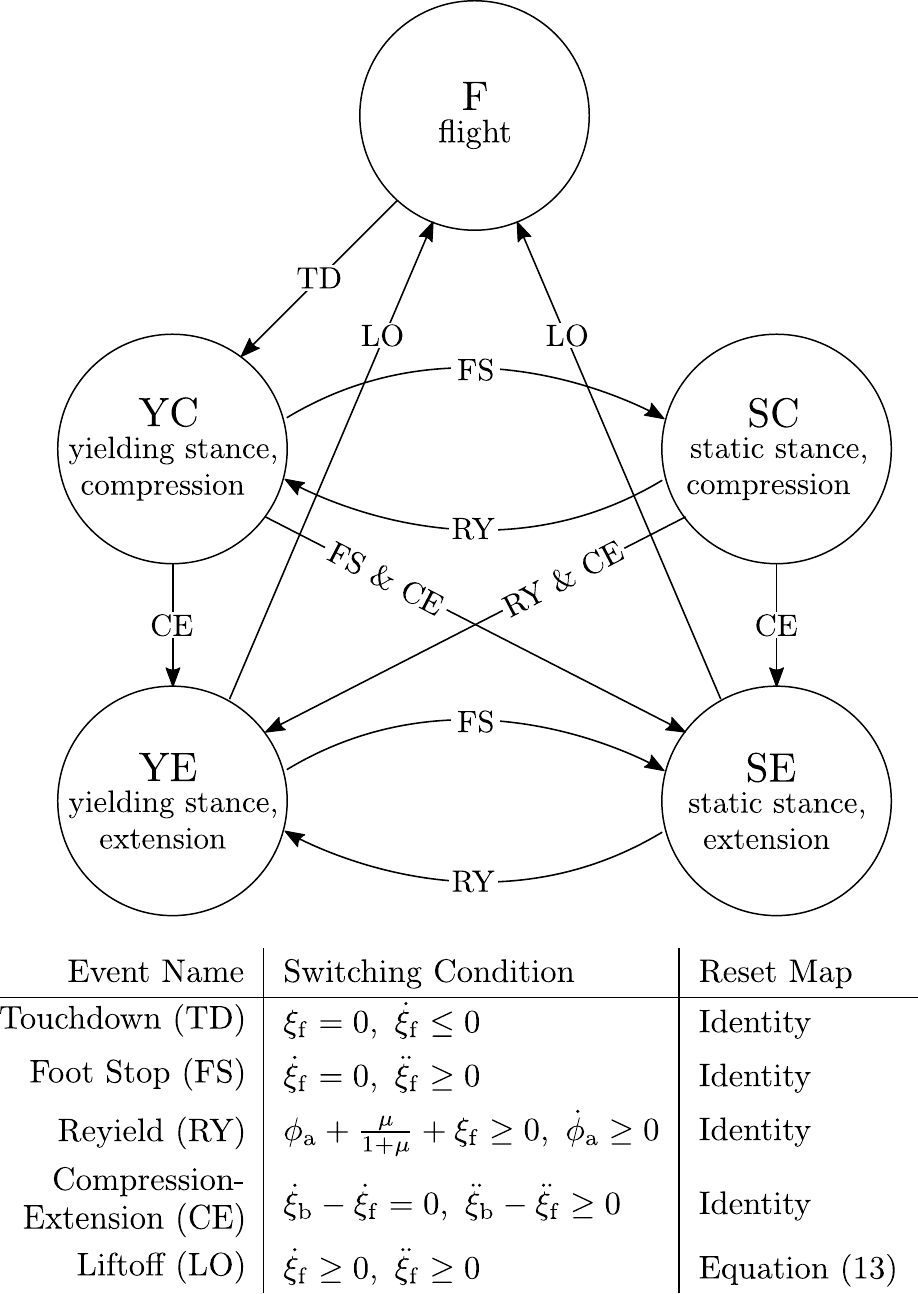}
    \caption{Directed graph illustrating all domains and all possible transitions for the hybrid system representation of the spring-legged monopod hopping vertically on deformable terrain.
    The table below the graph defines the switching condition and reset map for each event.
    }
    \label{fig:digraph}
\end{figure}

A \textit{gait} is a periodic solution $\mathcal{O}$ to the hybrid system that \strout{is transverse to the switching surface $\mathcal{S}_\mathrm{TD}$, such that the solution $\mathcal{O}$} passes once through stance phase and once through flight phase on every orbit\footnote{Requiring the orbit $\mathcal{O}$ to pass through the flight phase eliminates orbits in which the monopod is stationary as well as orbits in which the body oscillates up and down while the foot maintains ground contact.}, although stance phase may consist of multiple subphases (\eg yielding stance, static stance).
We use the notation $\mathcal{O}_\theta$ to represent a gait $\mathcal{O}$ that is parameterized by $\theta$, defined in Equation \eqref{eq:nondim_params_set}.
Figure \ref{fig:example_kinematics_and_hop_to_hop_energy_timeseries} shows an example trajectory, found by numerical integration beginning from initial conditions $\greekcolvec{\xi} = \left[\lambda_\mathrm{c},0,0,0\right]\transpose$, \ie beginning at rest and in contact with the substrate surface.
After several hops, the monopod settles into a period-one hopping gait with domain sequence $(\mathrm{YC},\mathrm{SC})\times 3,\mathrm{YE},\mathrm{SE},\mathrm{F}$.
For this particular combination of $\kappa_\mathrm{c}$ and $\mu$, reyielding occurs twice before energy is injected at time $\tau_\mathrm{CE}$ via a change of leg spring parameters.
For this particular force ratio ($\phi = 1.25$), reyielding also occurs at $\tau_\mathrm{CE}$ when energy is injected.

\section{Hop-to-Hop Energy Map}\label{sec:h2h_E_map}
\rev{In this section, we develop the \textit{hop-to-hop energy map} from the hybrid system model described in the previous section.
This approach follows the classical technique of using \Poincare sections and \Poincare return maps to analyze the existence and stability of limit cycles in nonlinear systems~\cite{HZDbook}.
Formulating the map in terms of the kinetic energy of the \ac{COM} (instead of its momentum or velocity) facilitates our focus on the dissipative effect of a depth-dependent yield threshold and its impact on hopping gaits.}

Given parameters $\theta$, defined in Equation \eqref{eq:nondim_params_set}, the parameterized hop-to-hop energy map is
\begin{equation}\label{eq:nondim_energy_map}
\mathcal{P}_\theta\left(\varepsilon_\mathrm{TD}\right) = \varepsilon_\mathrm{TD} + \varepsilon_\mathrm{inj}\left(\varepsilon_\mathrm{TD},\theta\right) - \varepsilon_\mathrm{loss}\left(\varepsilon_\mathrm{TD},\theta\right),
\end{equation}
where $\varepsilon_\mathrm{TD}$ is the \ac{COM} kinetic energy on \ac{TD}, $\varepsilon_\mathrm{inj}$ is the energy injected, and $\varepsilon_\mathrm{loss}$ is the energy lost.
\strout{Thus, $\mathcal{P}_\theta$ maps a one-dimensional subset of $\mathcal{S}_\mathrm{TD}$ to itself.}
Given the \ac{COM} kinetic energy on the $k$-th touchdown, one iteration of the map returns the \ac{COM} kinetic energy on the next touchdown:
\begin{equation}\label{eq:nondim_energy_map_one_iteration}
    \varepsilon_\mathrm{TD}\left[k+1\right] = \mathcal{P}_\theta\left(\varepsilon_\mathrm{TD}\left[k\right]\right).
\end{equation}
A period-$n$ hopping gait for a given $\theta$ is a period-$n$ fixed point
    \begin{equation}
        \varepsilon_\mathrm{TD}^{*,n} = \mathcal{P}_\theta^{n}\left(\varepsilon_\mathrm{TD}^{*,n}\right),
    \end{equation}
    where $\mathcal{P}_\theta^{n}$ is shorthand for $n$ iterations of the map $\mathcal{P}_\theta$.
In this paper, we restrict our focus to period-one fixed points unless otherwise specified (\ie in Section~\ref{sec:chaos}), and we drop the $n$ superscript when referring to period-one fixed points,~\ie $\varepsilon_\mathrm{TD}^{*,1}$ is simply written $\varepsilon_\mathrm{TD}^*$.

\begin{figure}[t]
    \centering
    \includegraphics[width=\linewidth]{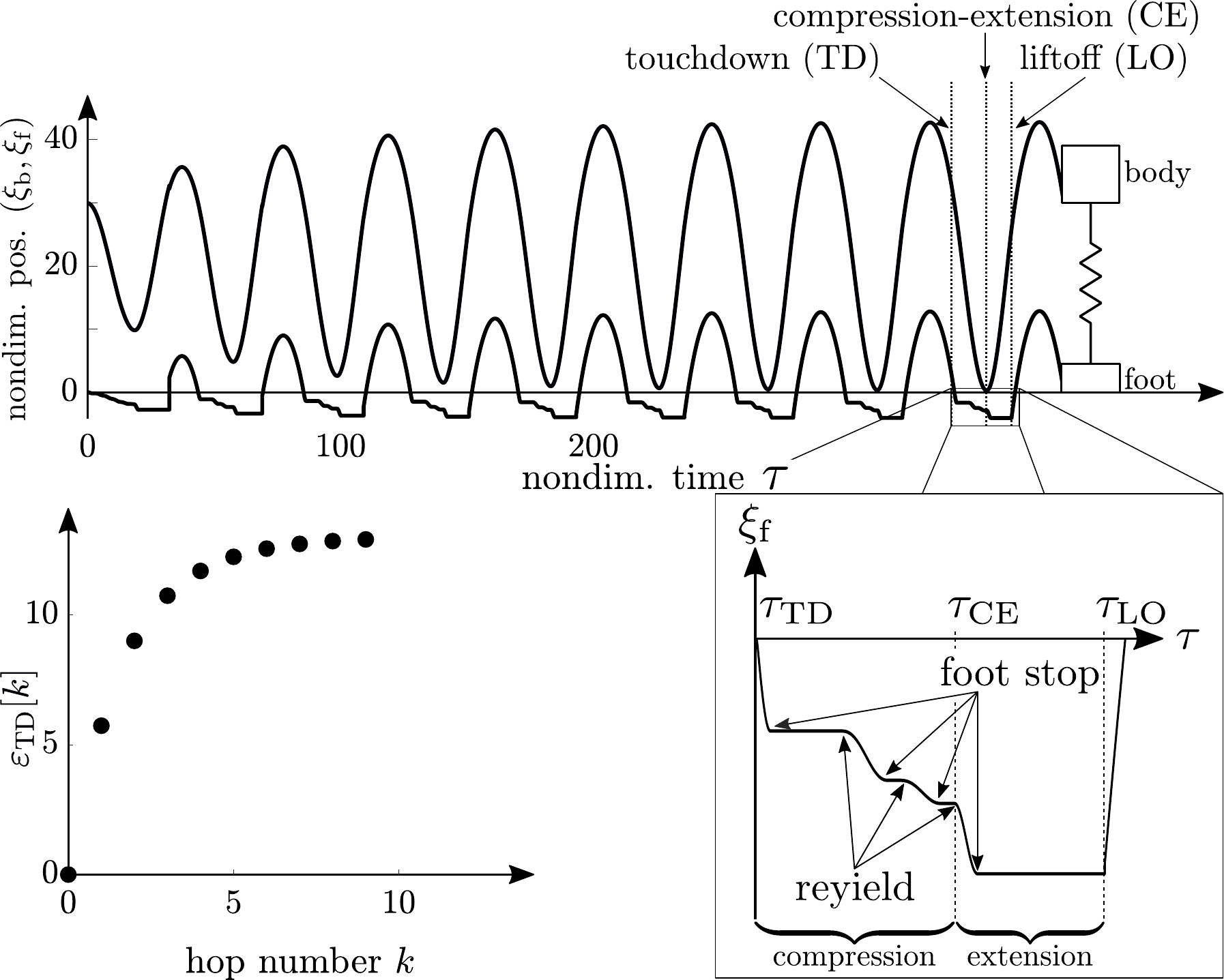}
    \caption{\textit{Top}: \rev{Illustrative plot of nondimensionalized body and foot trajectories $\xi_\mathrm{b}(\tau)$ and $\xi_\mathrm{f}(\tau)$, respectively, for $\varepsilon_\mathrm{inj}$ = 10, $\phi = 1.25$, $\kappa_\mathrm{c} = 0.1$, $\mu = 0.1$, $\lambda_\mathrm{c} = 35$, and initial conditions $\xi_\mathrm{b} = \lambda_\mathrm{c}$, $\xi_\mathrm{f} = 0$, $\dot{\xi}_\mathrm{b} = \dot{\xi}_\mathrm{f} = 0$.}
    \textit{Bottom left:} Nondimensionalized \ac{COM} kinetic energy at each \ac{TD}, $\varepsilon_\mathrm{TD}[k]$, for the monopod kinematics plotted above; $k$ is the hop number.
    \textit{Bottom right:} Detailed view of stance-phase nondimensionalized foot kinematics from the simulation above, for hop number $k = 8$. 
    }
    \label{fig:example_kinematics_and_hop_to_hop_energy_timeseries}
\end{figure}

In the remainder of this section, we describe experiments (Section \ref{sec:experiments}) and simulations (Section \ref{sec:simulations}) with a finite-foot-mass hopping monopod that show the hop-to-hop energy map captures the dynamics of a real \strout{monopedal robot}\rev{monopod} hopping on a real deformable substrate.
The close agreement between these experimental results and our intentionally simple model motivates our analysis of the dynamical properties of the model's hop-to-hop energy map in Section \ref{sec:dynamics_and_massless_foot_analysis}.

\subsection{Experiments}\label{sec:experiments}

\rev{We performed systematic drop-jump experiments with a two-mass vertically-constrained robot and a prepared bed of granular media.
These experiments validate the model from which the hop-to-hop energy map is derived and also capture the hop-to-hop evolution of the robot's \ac{COM} kinetic energy predicted by the map in the case of finite foot mass.}
\strout{To validate the dynamics predicted by the hop-to-hop energy map for a real robot hopping on a real substrate, we performed systematic drop-jump experiments with a two-mass vertically-constrained robot and a prepared bed of granular media.}

\subsubsection{Experimental Setup}
Our experimental apparatus, shown in Figure \ref{fig:apparatus}, consists of four components.
\begin{itemize}
    \item The first component is a vertically-constrained robot (see Figure \ref{fig:apparatus}, \rev{inset}) built around a linear brushless DC motor mounted on a linear ball bearing carriage.
    A rectangular acrylic foot with area $1.42 \times 10^{-2}$ m$^2$ is mounted at the bottom of the motor slider.
    The robot is instrumented with a sub-micron resolution absolute encoder that measures body position along the vertical guiderail, an incremental encoder that measures leg extension/compression, and a 6-axis force/torque sensor that measures ground reaction force.
    The motor stator, linear carriage, and absolute encoder count toward the body mass (2.5 kg) while the motor slider, force/torque sensor, and acrylic foot count toward the foot mass (0.5 kg), resulting in a mass ratio $\mu = 0.2$\rev{; this rather high mass ratio is a consequence of the minimalist construction of the ``robot'' whose body is only a motor, a lightweight encoder, and the hardware required to constrain the motor and encoder to the guiderail.}
    A 32-bit microcontroller reads \strout{these} sensors and implements the switched-compliance controller described in Section \ref{sec:controller} inside a 1 kHz control loop.
    \item The second component is a fluidized bed trackway filled to a depth of 19 cm with poppy seeds, a well-studied model granular substrate~\cite{Li_Terradynamics_2013,qian2015principles,jin2019preparation}.
    Air fluidization between sets of experiments ensures repeatable and homogeneous terrain conditions (\eg packing fraction and depth).
    \item The third component is a lifting mechanism that lifts and releases the robot from a user-specified height above the surface of the granular substrate.
    \item The fourth component is an x-y gantry that positions the robot above undisturbed soil before each drop-jump experiment.
\end{itemize}

\rev{
\begin{remark}
    Our experimental apparatus is similar to those used by us and others in previous studies on impacts and jumps on granular media~\cite{Aguilar_2016_Robophysical,Hubicki_Tractable_2016,roberts2018reactive,roberts2019mitigating,lynch2020SLP,chang2021learning}.
    The vertical guiderail enables the vertical hopping motion to be isolated without needing to stabilize the robot's rotational and horizontal translational degrees of freedom.
    While our experimental apparatus does not resemble most legged robots, it is suited to experimental validation of the soft-ground model, which is relevant to any legged robot locomotion on soft ground.
\end{remark}
}

\begin{figure}[t]
    \centering
    \includegraphics[width=0.7\linewidth]{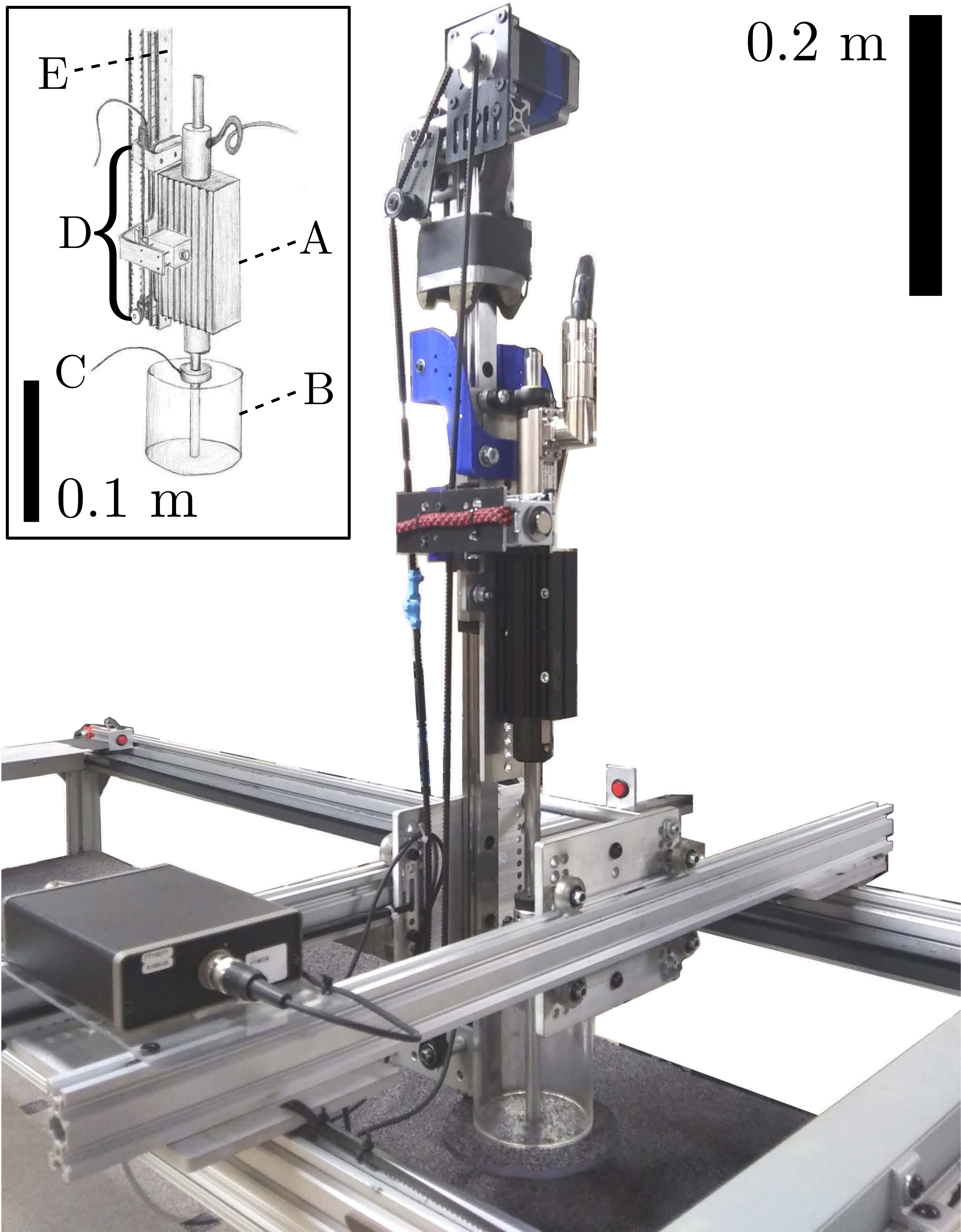}
    \caption{Experimental apparatus: vertically-constrained hopping robot mounted on an x-y gantry over a fluidized bed trackway filled with poppy seeds.
    Inset (from \cite{lynch2020SLP}): the robot consists of a motor stator (A), motor slider and acrylic foot (B), force/torque sensor (C), lifting mechanism (D), and vertical guiderail (E).}
    \label{fig:apparatus}
\end{figure}

\begin{figure*}[t]
    \centering
    \includegraphics[width=\linewidth]{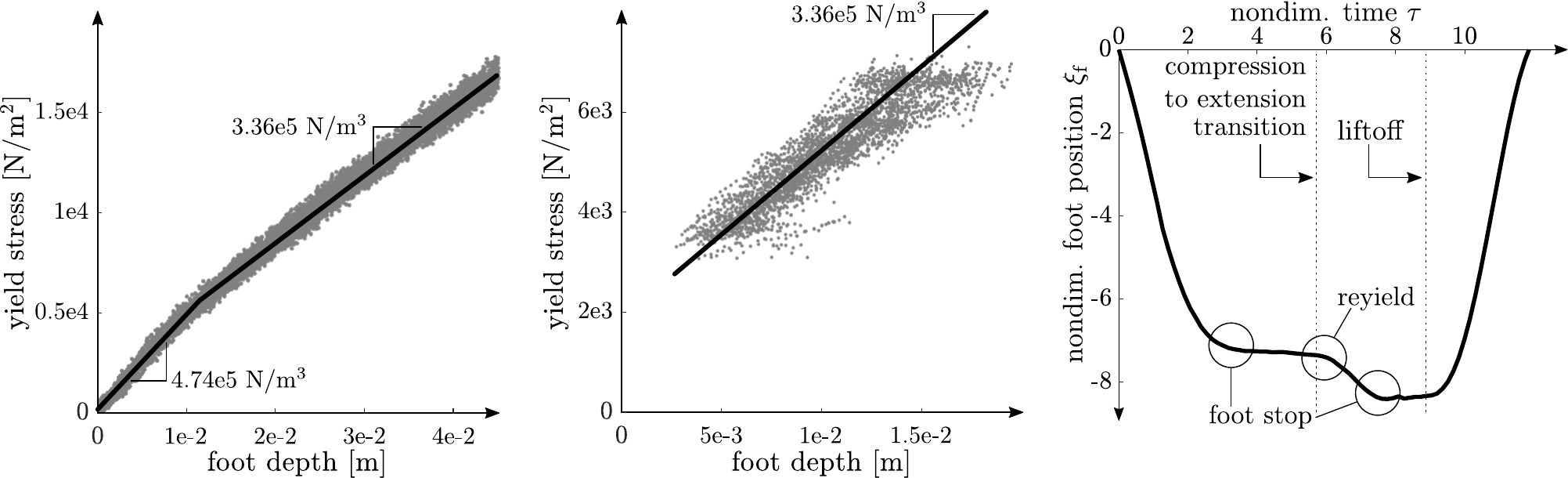}
    \caption{\textbf{Left:}~\rev{Experimental measurements of yield stress versus intruder depth for quasistatic intrusions (gray circles, foot area =  $2\times 10^{-3}$ m$^2$) show a surface regime (depth less than approximately 1 cm) and a bulk regime (depth greater than approximately 1 cm); in the bulk regime, the best-fit stress-per-unit-depth is approximately $3.36\times 10^{5}$ N/m$^3$.}
    \textbf{Middle:}~\rev{Measurements of stress versus depth recorded during jumping experiments (immediately prior to the compression-extension transition) agree well with the bulk-regime stress-per-unit depth estimate of $3.36\times 10^{5}$ N/m$^3$.}
    \textbf{Right:}~\rev{An example stance-phase foot trajectory from a jumping experiment ($\varepsilon_\mathrm{TD} = 3.8$, $\varepsilon_\mathrm{inj} = 20.3$, $\phi = 2$, $\kappa_\mathrm{c} = 0.27$, $\mu = 0.2$) demonstrates the foot-stop and reyield events predicted by our foot-ground interaction model.}
    }
    \label{fig:stress_vs_depth_and_reyield_example}
\end{figure*}
\subsubsection{Experimental Apparatus Characterization}
We characterize the granular substrate's penetration resistance by performing quasistatic intrusion experiments with a cylindrical intruder (5.1~cm diameter) while measuring penetration depth and \ac{GRF}.
By fitting a piecewise-linear curve to stress-versus-depth data, shown in Figure~\ref{fig:stress_vs_depth_and_reyield_example} (left), we determined ground stresses per unit depth of $4.74\times 10^{5}$ N/m$^3$ for depths less than approximately $10^{-2}$ m (the surface regime) and $3.36\times 10^{5}$ N/m$^3$ for depths greater than approximately $10^{-2}$ m (the bulk regime).
In other words, during quasistatic intrusion, the ground appears stiffer at shallow penetration depths than it does at deeper penetration depths.
These observations of two distinct penetration resistance regimes agree with previous work by Aguilar and Goldman~\cite{Aguilar_2016_Robophysical}, who attributed this change to the evolution of a cone of jammed granular media under the intruder.

To compare the effective ground stiffness during jumping experiments to the value estimated during quasistatic intrusions, we record the yield stress and foot penetration depth just prior to the compression-extension transition, when the foot is at rest before energy injection.
Figure~\ref{fig:stress_vs_depth_and_reyield_example} (middle) shows that the bulk-regime stress-per-unit depth ($3.36\times 10^{5}$ N/m$^3$) estimate from quasistatic intrusions agrees well with these yield-stress-versus-depth measurements despite the shallower penetration depth.
For the foot used in our jumping experiments, the effective ground stiffness is $k_\mathrm{g} = 4.8 \times 10^3$ N/m.
Because fluidization puts the granular bed in a loose-packed state, we hypothesize that, upon impact, the foot forces interstitial air out from between grains, resulting in localized fluidization that reduces the effective ground stiffness from the surface-regime value to a value closer to the bulk-regime stiffness \cite{jerome2016unifying}.
This agreement helps justify our model of deformable ground as a unidirectional spring that resists compression but not tension.

The right-hand plot in Figure~\ref{fig:stress_vs_depth_and_reyield_example} shows a stance-phase foot trajectory typical of our jumping experiments.
This trajectory displays the domain transitions predicted by our foot-ground interaction model: foot stop during compression, reyielding upon energy injection at the compression-extension transition, and foot stop again during extension.
This qualitative agreement further supports our intentionally simple foot-ground interaction model.

\subsubsection{Experimental Hop-to-Hop Energy Map Validation}
To experimentally validate the hop-to-hop energy map derived from our simple model, we prescribe $\varepsilon_\mathrm{inj}$, $\phi$, and $\kappa_\mathrm{c}$, and we drop the robot from a range of heights (foot-ground clearance ranges from 1~cm to 10~cm, in 0.5~cm increments), resulting in a range of $\varepsilon_\mathrm{TD}$.
\rev{The kinetic energy of the \ac{COM} at \ac{TD} is related to the drop height through conservation of energy:}
\strout{Given the apex \ac{COM} gravitational potential energy and the coefficient of Coulomb friction between the robot body and the vertical guiderail, the \ac{COM} kinetic energy at \ac{TD} is}
\rev{\begin{align}\label{eq:experiments_apex_KE_to_TD_KE}
    \begin{split}
        \varepsilon_\mathrm{TD} &= \varepsilon_\mathrm{apex} - \Delta\varepsilon_\mathrm{friction}\\
        &= \xi_\mathrm{CoM,apex} - \phi_\mathrm{c}\left(\xi_\mathrm{b,apex} - \xi_\mathrm{b,TD}\right),
    \end{split}
\end{align}
where $\varepsilon_\mathrm{apex}$ is the apex gravitational potential energy of the \ac{COM} (equivalent to nondimensionalized apex \ac{COM} position \rev{$\xi_\mathrm{CoM,apex}$}) and $\Delta\varepsilon_\mathrm{friction}$ is the work done by friction between the body and the guidrail during the drop, equivalent to $\phi_\mathrm{c}\left(\xi_\mathrm{b,apex} - \xi_\mathrm{b,TD}\right)$, where $\phi_\mathrm{c} = 0.15$ is the nondimensionalized Coulomb friction coefficient (estimated by fitting to \ac{COM} free-fall kinematics), and $\xi_\mathrm{b,apex}$ and $\xi_\mathrm{b,TD}$ are the body positions at the apex and at touchdown, respectively.
We use the same procedure to compute $\mathcal{P}_\theta(\varepsilon_\mathrm{TD})$, the \ac{COM} kinetic energy on the next \ac{TD}, given the \ac{COM} kinetic energy at \ac{LO}.
This procedure produces a $\left(\varepsilon_\mathrm{TD},\mathcal{P}_\theta\left(\varepsilon_\mathrm{TD}\right)\right)$ pair that accounts for energy dissipation due to friction in the guiderail.
}

We systematically vary $\varepsilon_\mathrm{inj}$ from 13.6 to 27.1 while holding $\phi$ and $\kappa_\mathrm{c}$ constant at 1.5 and 0.27, respectively.
For each drop height and combination of independent parameters, we perform ten drop-jump experiments, spaced evenly over the length and width of the granular bed, after which we fluidize the bed to prepare the soil for the next set of experiments.
This experimental procedure generates 190 $(\varepsilon_\mathrm{TD}, \mathcal{P}_\theta(\varepsilon_\mathrm{TD}))$ pairs per parameter combination, totaling 950 experimental data points across five values of $\varepsilon_\mathrm{inj}$.

\begin{figure}[t]
    \centering
    \includegraphics[width=\linewidth]{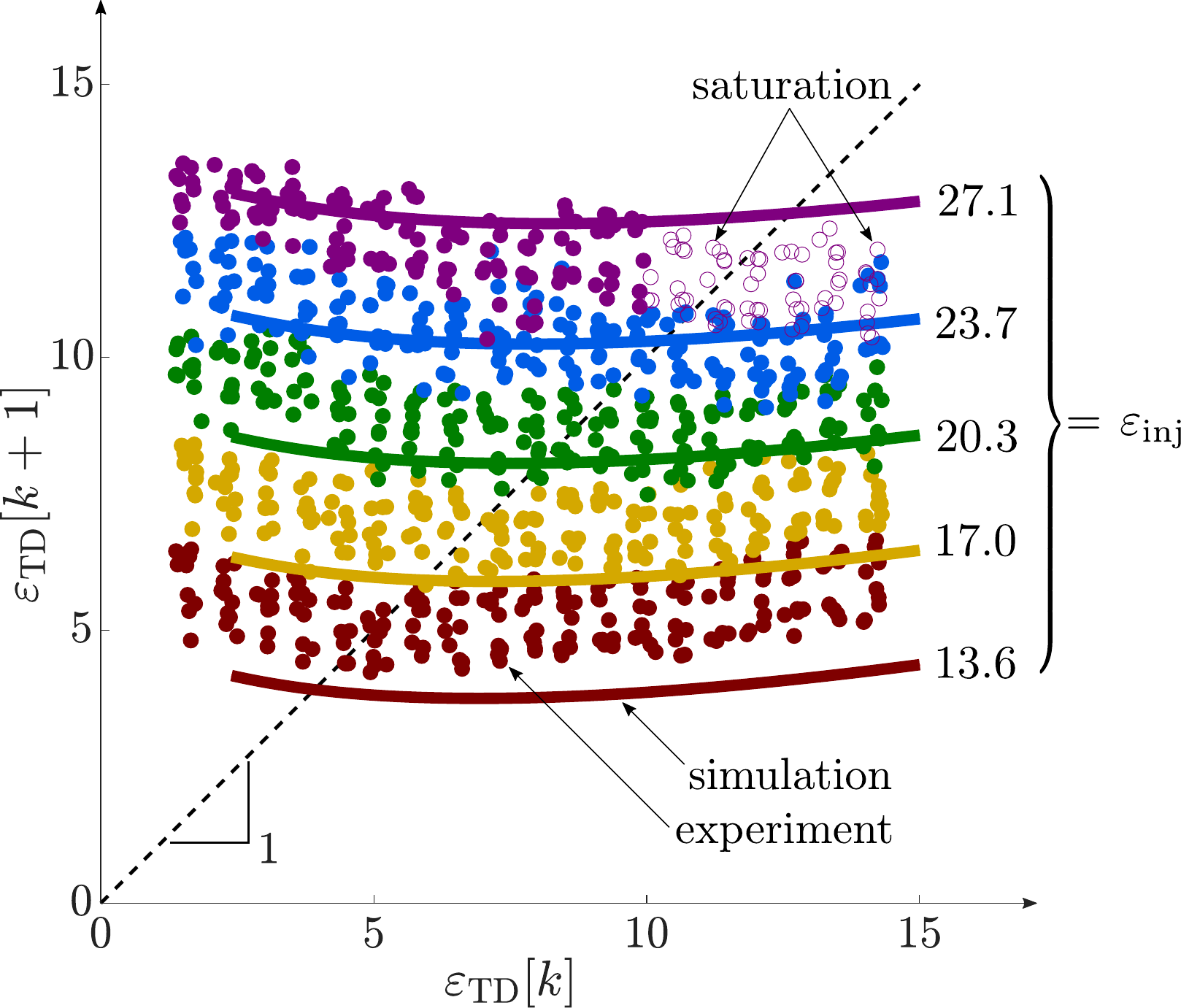}
    \caption{Comparison between experiments and simulations: circles represent experimental $(\varepsilon_\mathrm{TD},\mathcal{P}_\theta(\varepsilon_\mathrm{TD}))$ pairs computed from measured kinematic data using Equation \eqref{eq:experiments_apex_KE_to_TD_KE}; curves are hop-to-hop energy maps $\mathcal{P}_\theta$ simulated using experimental parameters.
    Hopping gaits (fixed points of the energy map) correspond to points where the data intersect the unity map, indicated by a thin black line.
    Motor force saturation (indicated by hollow circles) occurs at large $\varepsilon_\mathrm{TD}[k]$ and large commanded $\varepsilon_\mathrm{inj}$, resulting in reduced injected energy compared to the prescribed value.}
    \label{fig:experimental_results}
\end{figure}

Figure~\ref{fig:experimental_results} compares the results of our jumping experiments to curves from simulations \rev{of the hybrid dynamics described in Section~\ref{sec:hybrid_model} with an additional term included to account} for Coulomb friction between the body and the guiderail.
All simulations are performed using MATLAB's adaptive-timestep integrator \texttt{ode15s}.
The experimental data agree qualitatively with the hop-to-hop energy maps predicted by simulation, despite several unmodeled aspects of granular intrusion, including inertial effects due to granular accretion underfoot \cite{Aguilar_2016_Robophysical}, free-surface effects, and elasticity of the grains themselves \cite{jaeger1996granular}, as well as edge and bottom effects induced by the walls of the fluidized bed.
The fact that experiments with a real \strout{robot}\rev{monopod} hopping on a real substrate agree with our deliberately simple robophysical model points to its value as a basis for planning and controlling legged robot locomotion on real deformable substrates.

\subsection{Simulations}\label{sec:simulations}
The domain sequence in a hopping gait can be quite complex, depending on $\mu$ and $\kappa_\mathrm{c}$, as illustrated in Figure~\ref{fig:example_kinematics_and_hop_to_hop_energy_timeseries}.
While this complexity precludes closed-form representation of $\mathcal{P}_\theta$ for $\mu > 0$, a closed-form representation exists in the limit of $\mu\to0$ (see Appendix for derivation).
In the remainder of this section, we use simulations as a bridge between the finite-foot-mass hopping experiments described above and our analysis of the dynamics of the hop-to-hop energy map in the massless-foot limit, described later in Section~\ref{sec:massless_foot_analysis}.

\subsubsection{Comparison between Massless-Foot and Finite-Mass-Foot Dynamics}\label{sec:compare_massless_and_finite_foot_mass_maps}
We numerically simulate the hybrid system model of the monopod \strout{(see Section \ref{sec:hybrid_model})} for mass ratios $\mu$ from \strout{0.001} \rev{$10^{-3}$} to 1.
These simulations calculate the dissipative effect of finite foot mass and also verify the accuracy of our massless-foot analysis\strout{, which we describe later in Section \ref{sec:massless_foot_analysis}}.
In Figure \ref{fig:sim_overlay_massless_soln}, we plot the \strout{resulting} fixed point of the hop-to-hop energy map versus \rev{the} mass ratio and \rev{the} force ratio, for \strout{injected energy} $\varepsilon_\mathrm{inj} = 10$ and \strout{stiffness ratio} $\kappa_\mathrm{c} = 0.1$.
The massless-foot solution (the black curve overlaid on the surface at the lowest-simulated mass ratio, $\mu = 10^{-3}$) is in close agreement with the finite-foot-mass simulations as the mass ratio approaches zero.
Broadly speaking, for a given $\varepsilon_\mathrm{inj}$, $\phi$, and $\kappa_\mathrm{c}$, $\varepsilon_\mathrm{TD}^*$ decreases approximately monotonically as $\mu$ increases, agreeing with Equation \eqref{eq:nondim_energy_loss_at_liftoff}.
\rev{As $\mu$ increases, both $\varepsilon_\mathrm{TD}^*$ and $\Lambda$ decrease, although the finite-foot mass map is quite similar to the massless-foot map for $\mu = 10^{-3}$ and $\mu = 10^{-2}$, and even the $\mu = 10^{-1}$ case is still relatively similar to the massless-foot map, as shown in Figure~\ref{fig:sim_overlay_massless_soln}.

\begin{figure}[t]
    \centering
    \includegraphics[width=\linewidth]{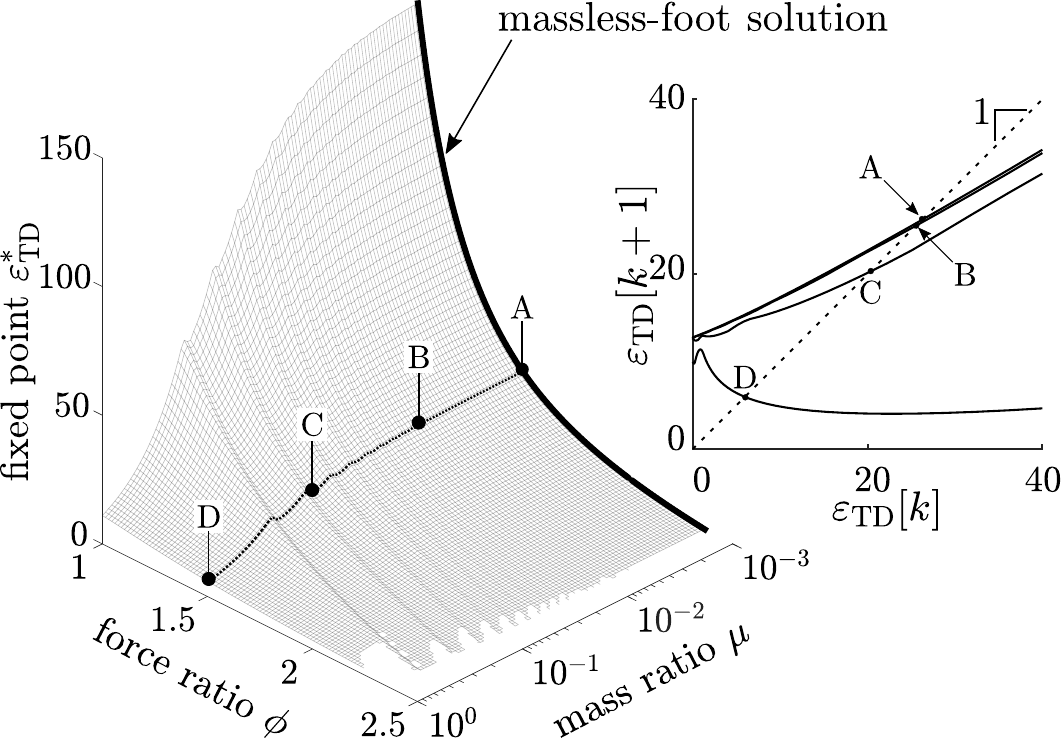}
    \caption{
    Comparison of fixed points $\varepsilon_\mathrm{TD}^*$ of map $\mathcal{P}_\theta$ for negligible foot mass ($\mu\to 0$, black curve) and finite foot mass ($\mu > 0$, surface), with $\varepsilon_\mathrm{inj} = 20$ and $\kappa_\mathrm{c} = 0.1$.
    \textit{Left:~}
    \rev{The closed-form negligible-foot-mass map (black curve) closely matches the numerically computed finite-foot-mass map (surface) at $\mu = 10^{-3}$, as discussed in Section~\ref{sec:compare_massless_and_finite_foot_mass_maps}.}
    \textit{Right:~}
    Computed maps for $\phi = 1.5$ and four mass ratios: $\mu = 10^{-3}$ (A), $\mu = 10^{-2}$ (B), $\mu = 10^{-1}$ (C), and $\mu = 10^{0}$ (D).
	}
    \label{fig:sim_overlay_massless_soln}
\end{figure}

\begin{remark}
This mass ratio range ($10^{-3}$ to $10^{-1}$) encompasses the 1:18 ratio of unsprung mass to sprung mass of Raibert's 3D hopping monopod and the mass ratios of Raibert's quadruped~\cite{raibert1986legged}.
Quadrupedal robots like the MIT Mini Cheetah and commercially-available counterparts also belong to this category of legged robots with low mass ratios, as their hip and knee actuators are located close to the torso to minimize leg inertia~\cite{katz2019cheetah}.
\end{remark}
}

\strout{There is little change in $\varepsilon_\mathrm{TD}^*$ as $\mu$ increases from $10^{-3}$ to $10^{-2}$.}
\rev{As shown in Figure~\ref{fig:sim_overlay_massless_soln}, local extrema in the finite-foot-mass energy map emerge}
from the interplay between the forces applied during impact by the leg spring and the depth-dependent yield threshold of the terrain underfoot, both of which depend on the kinetic energy of the robot \ac{COM} on impact.
\rev{These extrema, with magnitudes that grow with $\mu$, change both the fixed point and the eigenvalue of the finite-foot-mass map.
These changes are relatively small except when both $\mu$ and $\phi$ are large, in which case the transient response of the finite-foot-mass map can be qualitatively different from that of the massless-foot map.
However, we emphasize that these qualitative differences in the hop-to-hop energy dynamics are restricted to the combination of a large force ratio and a heavy foot.}
\rev{Nevertheless, this feature indicates} that finite foot mass can result in more complex hopping behavior than suggested by the massless-foot hop-to-hop energy map discussed in \strout{Section \ref{sec:massless_foot_analysis}} \rev{the next section}.
In the following paragraphs, we discuss two such examples of complex behavior: period-two hopping gaits and chaotic hopping.

\subsubsection{Bifurcations and Chaos Near the Minimum Viable Gait}\label{sec:chaos}
In simulation, sweeping $\varepsilon_\mathrm{inj}$ while holding all other parameters constant reveals period-one gaits, period-two gaits, and a cascade of period-doubling bifurcations leading to chaos, illustrated in Figure \ref{fig:bif_diagram_massive_foot}.
As $\varepsilon_\mathrm{inj}$ is increased from $\varepsilon_\mathrm{inj,min}\approx 11.68$, a subcritical bifurcation occurs, followed by a cascade of period-doubling bifurcations leading to chaos around $11.85 < \varepsilon_\mathrm{inj} < 11.9$.
This chaotic region is relatively contained, with $\varepsilon_\mathrm{TD}^*$ bounded (approximately) between 0.1 and 0.3.
As $\varepsilon_\mathrm{inj}$ increases above 11.9, a cascade of period-halving bifurcations occurs, culminating in period-one hopping for $12.5 < \varepsilon_\mathrm{inj} < 13.5$, after which there follows another period-two region for $13.5 < \varepsilon_\mathrm{inj} < 14.1$.
For $\varepsilon_\mathrm{inj} > 14.1$, only period-one fixed points exist.
The massless-foot solution, indicated by a solid black line, approximately agrees with period-one fixed points of the map for $\mu = 0.02$.

Such complex behavior is expected, given the period-doubling route to chaos demonstrated by a Raibert-like hopper on hard ground \cite{vakakis1991strange} and the general similarity to the bouncing ball problem, a well-studied chaotic dynamical system \cite{guckenheimer2013nonlinear}.
Furthermore, this behavior supports our hypothesis that complex dynamics are possible on plastically deformable terrain, even for as simple a robot as our monopod using a Raibert-like switched-compliance energy injection controller.
This chaotic behavior may be undesirable, as sensitive dependence on initial conditions makes it difficult to predict future values of $\varepsilon_\mathrm{TD}$.
In this case, the energy map can be analyzed to determine ``no-go'' regions in parameter space that place $\varepsilon_\mathrm{TD}$ within the \ac{BOA} of the strange attractor responsible for chaotic behavior.
However, chaotic hopping may \rev{also offer benefits}:~\strout{be beneficial:} once on the strange attractor, $\varepsilon_\mathrm{TD}$ takes on many values, bounded above and below, thus providing an information-rich stimulus from which to estimate ground stiffness.
It may also be possible to exit the strange attractor in a single hop once a desired fixed point \strout{$\varepsilon_\mathrm{TD}^*$} is reached by altering controllable parameters such that the eigenvalue \strout{$\Lambda$} equals zero while the fixed point \strout{$\varepsilon_\mathrm{TD}^*$} is unchanged.

\begin{figure}[t]
    \centering
    \includegraphics[width=\linewidth]{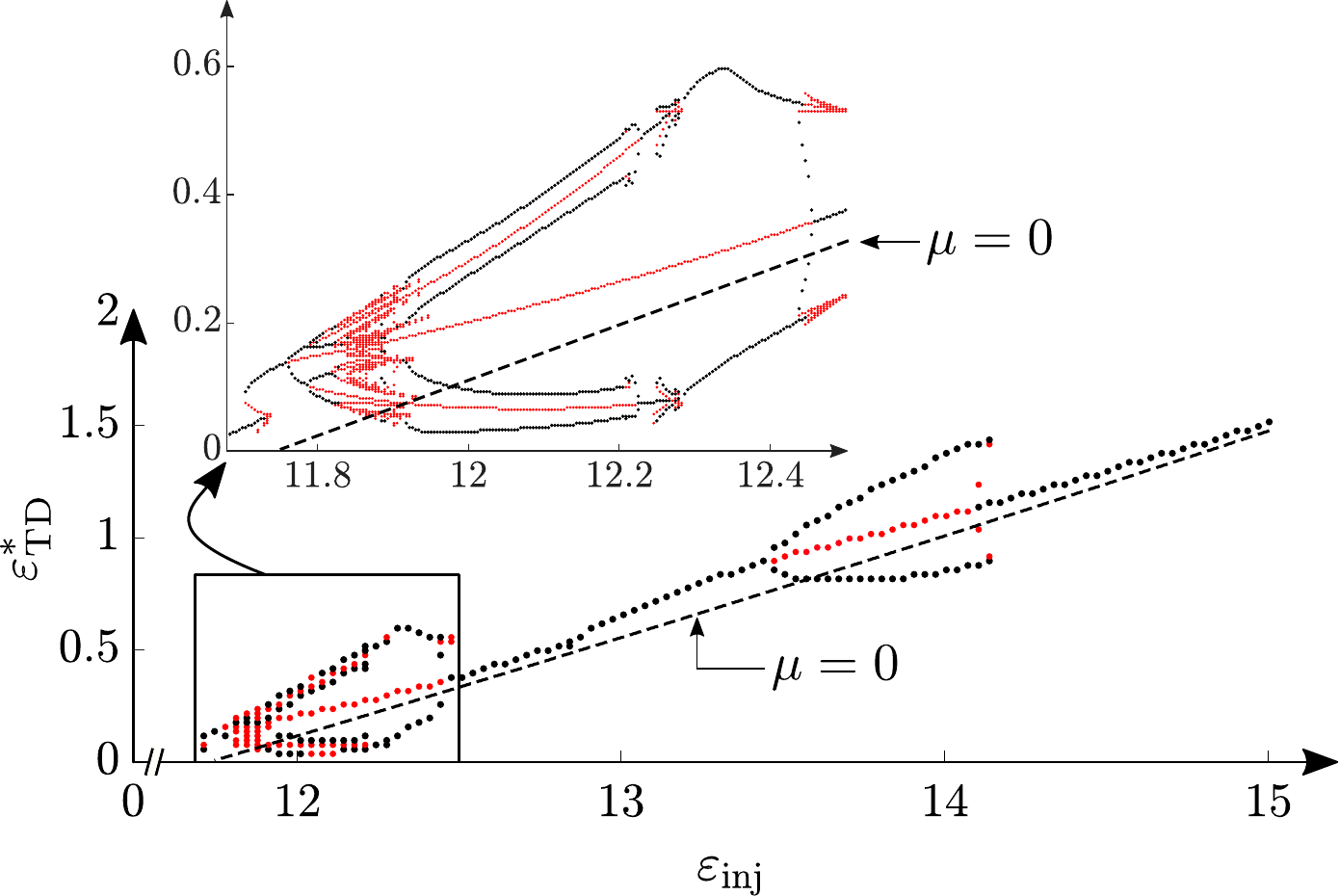}
    \caption{Bifurcation diagram: $\varepsilon_\mathrm{TD}^*$ versus $\varepsilon_\mathrm{inj}$; $\mu = 0.02$, $\kappa_\mathrm{c} = 0.25$, $\phi = 2$.
    The bifurcation diagram is generated using the algorithm described in \cite{ross2000bifurcation} with resolutions of $100\times100$ (main diagram) and $200\times200$ (inset).
    Black circles indicate linearly stable fixed points and \rev{red} circles indicate linearly unstable fixed points.
    The \rev{dashed} black line indicates the massless-foot solution ($\mu = 0$).
    }
    \label{fig:bif_diagram_massive_foot}
\end{figure}

\section{Hop-to-Hop Energy Dynamics}\label{sec:dynamics_and_massless_foot_analysis}
Having established in the previous section the applicability of the hop-to-hop energy map to a real robot hopping on a real deformable substrate, 
we turn our attention to \strout{defining and discussing} several fundamental locomotion qualities---efficiency, \rev{energy} stability margin (a proxy for agility), and robustness---\rev{which we define} in terms of properties of the hop-to-hop energy map.
\rev{We then derive a closed-form expression for the hop-to-hop energy map in the limit of zero foot mass, which, in turn, enables us to discuss efficiency, agility, and robustness in terms of dimensionless control parameters.

First, we show that there is a 2D surface in 3D control parameter space for which any parameter combination produces the same steady-state hopping gait, while different combinations result in different values for efficiency and the energy stability margin.
In particular, efficient hopping results from low stiffness ratios and low force ratios, as these reduce the energy dissipated via ground deformation underfoot, while a higher energy stability margin (\ie a faster transient response) results from higher stiffness ratios or higher force ratios, combined with injecting more energy per hop.
This relationship captures an intuitive tradeoff between energetic efficiency and agility.

We also develop guidelines for choosing control parameters to ensure robust hopping, based on the basin of attraction of the map's fixed point.
Specifically, in the limit of zero foot mass, we prove that the fixed point of the hop-to-hop energy map is globally stable if the map's eigenvalue is between zero and one.
We also find necessary and sufficient conditions for the hop-to-hop energy map to have a banded basin structure, a feature which is undesirable from a robustness perspective; these conditions also offer guidance for choosing control parameters to ensure robust hopping.
}

\subsection{Local Properties: Efficiency and Energy Stability Margin}\label{sec:local_map_properties}
We begin our discussion of dynamic locomotion properties with two local gait characteristics: efficiency and \rev{energy} stability margin.
\begin{definition}[Efficiency]\label{def:efficiency}
    Given a period-one fixed point $\varepsilon_\mathrm{TD}^*$ of a map $\mathcal{P}_\theta$, we define the \textit{efficiency} $\eta$ of the corresponding gait $\mathcal{O}_\theta$ as
    \begin{equation}\label{eq:efficiency_def}
        \eta = \frac{\varepsilon_\mathrm{TD}^*}{\varepsilon_\mathrm{inj} + \varepsilon_\mathrm{TD}^*},
    \end{equation}
    where $\varepsilon_\mathrm{inj}$ is the energy injected per hop.
\end{definition}
\begin{remark}
    Efficiency is bounded between 0 and 1, but irrecoverable energy loss due to ground deformation means some energy must be injected on every hop, so $\eta = 1$ is unattainable in practice.
\end{remark}
The next definition follows from properties of the eigenvalue $\Lambda$ of the hop-to-hop energy map $\mathcal{P}_\theta$, given by
\begin{equation}\label{eq:eigenvalue_def}
        \Lambda = \frac{\mathrm{d}\mathcal{P}_\theta}{\mathrm{d}\varepsilon_\mathrm{TD}}\bigg\vert_{\varepsilon_\mathrm{TD}^*}
        = 1 - \left(\frac{\partial\varepsilon_\mathrm{loss}}{\partial\xi_\mathrm{f^+}}\frac{\partial\xi_\mathrm{f^+}}{\partial\xi_\mathrm{f^-}}\frac{\partial\xi_\mathrm{f^-}}{\partial\varepsilon_\mathrm{TD}}\right) \bigg\vert_{\varepsilon_\mathrm{TD}^*},
\end{equation}
where $\partial\varepsilon_\mathrm{loss}/\partial\xi_\mathrm{f^+}$, $\partial\xi_\mathrm{f^+}/\partial\xi_\mathrm{f^-}$, and $\partial\xi_\mathrm{f^-}/\partial\varepsilon_\mathrm{TD}$ depend on parameters $\theta$.
The eigenvalue $\Lambda$ depends on the fixed point $\varepsilon_\mathrm{TD}^*$ and parameters $\theta$ and quantifies the local stability of the fixed point $\varepsilon_\mathrm{TD}^*$ of the map $\mathcal{P}_\theta$ and therefore the local stability of the corresponding gait $\mathcal{O}_\theta$.
A gait $\mathcal{O}_\theta$ is locally \textit{stable} if the eigenvalue $\Lambda$ associated with the corresponding fixed point $\varepsilon_\mathrm{TD}^*$ of the map $\mathcal{P}_\theta$ has magnitude less than one.
The eigenvalue $\Lambda$ also characterizes the transient response, including the rate at which nearby initial conditions $\varepsilon_\mathrm{TD}[0]$ converge to $\varepsilon_\mathrm{TD}^*$ and whether this convergence is monotonic ($\Lambda > 0$), oscillatory ($\Lambda < 0$), or deadbeat ($\Lambda = 0$).
\begin{definition}[\rev{Energy} Stability Margin]\label{def:stability_margin}
    We define the \textit{\rev{energy} stability margin} $\alpha$ as
    \begin{equation}\label{eq:stability_margin}
        \alpha =
        \begin{cases}
            1 - \Lambda^2,&\textrm{ if } \Lambda \in \left[-1,1\right],\\
            0,&\textrm{ otherwise},
        \end{cases}
    \end{equation}
    where the eigenvalue $\Lambda$ is defined in Equation~\eqref{eq:eigenvalue_def}.
\end{definition}
\begin{remark}
    \rev{Rapid} convergence to a \rev{stable} target \ac{COM} kinetic energy \rev{is} relevant to agile locomotion \cite{full2002quantifying,sheppard2006agility}.
    The proposed definition of the \rev{energy} stability margin captures the fact that an eigenvalue closer to zero results in a faster transient response (convergence in fewer hops), compared to an eigenvalue closer to $1$ or $-1$, and it also captures the fact that \strout{hop-to-hop} local stability \rev{of the fixed-point} ($\Lambda \in \left[-1,1\right]$) is a prerequisite for agility.
\end{remark}

\subsection{Global Properties: Basins of Attraction as Criteria for Robustness}\label{sec:global_map_properties}
In general, the hop-to-hop energy map given in Equation~\eqref{eq:nondim_energy_map} is a nonlinear function of $\varepsilon_\mathrm{TD}$, and therefore, while the map may have a locally stable fixed point, convergence from arbitrary initial conditions is not guaranteed.

\begin{definition}[Basin of Attraction]\label{def:BOA}
The \textit{basin of attraction} (BOA) of a fixed point $\varepsilon_\mathrm{TD}^*$ of the parameterized hop-to-hop energy map $\mathcal{P}_\theta$ is the set of initial conditions $\varepsilon_\mathrm{TD}$ that converge to the fixed point:
\begin{equation}
    \mathrm{BOA} = \left\{\varepsilon_\mathrm{TD}\in\left[0,\infty\right) \mid \lim_{k\to\infty}\mathcal{P}_\theta^k\left(\varepsilon_\mathrm{TD}\right) = \varepsilon_\mathrm{TD}^*\right\}.
\end{equation}
The \textit{immediate} \ac{BOA} is the largest interval containing the fixed point $\varepsilon_\mathrm{TD}^*$ that lies in the \ac{BOA} \cite{devaney2021introduction}.
\end{definition}

\begin{remark}
    In this work, we characterize the \ac{BOA} and its dependence on model parameters $\varepsilon_\mathrm{inj}$, $\phi$, and $\kappa_\mathrm{c}$.
    Under certain conditions, the \ac{BOA} and the immediate \ac{BOA} are both the semi-open interval $\left[0,\infty\right)$; in this case, we say $\varepsilon_\mathrm{TD}^*$ is \textit{globally stable}.
    Under other conditions, the \ac{BOA} has a more complex banded structure, and the immediate \ac{BOA} is a subset of the full \ac{BOA}.
    We discuss both cases \strout{in Section \ref{sec:BOA_massless}}\rev{in the remainder of this section}.
\end{remark}

\subsection{Hop-to-Hop Energy Map Analysis in the Limit of Negligible Foot Mass}\label{sec:massless_foot_analysis}

In the remainder of this section, we analyze the nondimensionalized hop-to-hop energy map in the limit of negligible foot mass for which the map has a closed-form expression.
In the limit of $\mu\to 0$, the closed-form representation of the hop-to-hop energy map is
\begin{align}\label{eq:nondim_massless_foot_energy_map}
    \mathcal{P}_\theta\left(\varepsilon_\mathrm{TD}\right) &= \varepsilon_\mathrm{TD} + \varepsilon_\mathrm{inj} - \varepsilon_\mathrm{inj}\left(\frac{a \xi_\mathrm{f^-}^3 + c \xi_\mathrm{f^-}}{b \xi_\mathrm{f^-}^2 + d}\right)^2,
\end{align}
where $\xi_\mathrm{f^-}$ is shorthand for \[\xi_\mathrm{f}(\tau_\mathrm{CE}^-) = -(1 + \sqrt{1 + 2 \kappa_{c} \varepsilon_\mathrm{TD} (1 + \kappa_\mathrm{c})^{-1}}),
\]
\ie the foot penetration depth immediately prior to energy injection, and $a$, $b$, $c$, and $d$ are positive coefficients that depend on parameters $\theta$.
See the Appendix for derivations of Equation \eqref{eq:nondim_massless_foot_energy_map} and coefficients $a$, $b$, $c$, and $d$.

\subsubsection{Fixed Points}\label{sec:fixed_points_massless}
We use monotonicity of $\varepsilon_\mathrm{loss}$ with respect to $\varepsilon_\mathrm{TD}$ to prove that there is a unique fixed point $\varepsilon_\mathrm{TD}^*$ for each combination of $\varepsilon_\mathrm{inj}$, $\phi$, and $\kappa_\mathrm{c}$.
We then solve for $\varepsilon_\mathrm{TD}^*$ given $\varepsilon_\mathrm{inj}$, $\phi$, and $\kappa_\mathrm{c}$ and use this solution to characterize the impact of these dimensionless parameters on the fixed point.

\begin{lemma}\label{lemma:monotonic_eloss}
    In the limit $\mu\to0$, $\varepsilon_\mathrm{loss}$ increases monotonically with $\varepsilon_\mathrm{TD}$.
\end{lemma}
\begin{proof}
    Direct computation of $\mathrm{d}\varepsilon_\mathrm{loss}/\mathrm{d}\varepsilon_\mathrm{TD}$ shows that the sign of $\mathrm{d}\varepsilon_\mathrm{loss}/\mathrm{d}\varepsilon_\mathrm{TD}$ is determined by the sign of the quadratic polynomial $y = a b x^2 + (3 a d - b c)x + c d$, where $x = \xi_\mathrm{f^-}^2$.
    For real $\xi_\mathrm{f^-}$, the minimum possible value of $y$ is $c d$ (corresponding to $\xi_\mathrm{f^-} = 0$), which is positive for all positive $\varepsilon_\mathrm{inj}$, $\phi$, and $\kappa_\mathrm{c}$.
    Because $y$ increases monotonically with $x$ for $x \geq 0$, $y$ is positive for $\xi_\mathrm{f^-} \geq 0$.
    Therefore $\mathrm{d}\varepsilon_\mathrm{loss}/\mathrm{d}\varepsilon_\mathrm{TD} > 0$ for all $\varepsilon_\mathrm{TD} \geq 0$, so $\varepsilon_\mathrm{loss}$ increases monotonically with $\varepsilon_\mathrm{TD}$.
\end{proof}

\begin{remark}
    While $\varepsilon_\mathrm{loss}$ increases monotonically with $\varepsilon_\mathrm{TD}$ in the limit $\mu\to0$, this monotonic relationship does not hold for sufficiently large $\mu$, as discussed in Section \ref{sec:compare_massless_and_finite_foot_mass_maps}.
\end{remark}

\begin{proposition}\label{prop:unique_fxpt}
    In the limit $\mu\to0$, the map $\mathcal{P}_\theta$ has a unique fixed point $\varepsilon_\mathrm{TD}^*$.
\end{proposition}
\begin{proof}
    The proof follows automatically from monotonicity of $\varepsilon_\mathrm{loss}$ with respect to $\varepsilon_\mathrm{TD}$ (see Lemma \ref{lemma:monotonic_eloss}) and the fact that $\varepsilon_\mathrm{inj}$ is constant.
\end{proof}

We find period-one fixed points of the parameterized map $\mathcal{P}_\theta$ in the limit of $\mu\to 0$ by solving Equation~\eqref{eq:nondim_massless_foot_energy_map} for $\varepsilon_\mathrm{TD}^*$, which simplifies to 
\begin{equation}\label{eq:cubic_pre_inj_depth_fxpt}
    a \xi_\mathrm{f^-}^3 + b \xi_\mathrm{f^-}^2 + c \xi_\mathrm{f^-} + d = 0.
\end{equation}

\begin{remark}
    For all physically plausible $\varepsilon_\mathrm{inj}$, $\phi$, and $\kappa_\mathrm{c}$, the discriminant of the cubic polynomial in Equation \eqref{eq:cubic_pre_inj_depth_fxpt} is negative, indicating one real solution for $\xi_\mathrm{f^-}$, which agrees with Proposition \ref{prop:unique_fxpt}.
\end{remark}

Figure~\ref{fig:fxpt_surf_massless_foot_kc=0.1} illustrates the relationship between $\varepsilon_\mathrm{TD}^*$ and parameters $\varepsilon_\mathrm{inj}$, $\phi$, and $\kappa_\mathrm{c}$.
Increasing $\varepsilon_\mathrm{inj}$ while holding all other parameters constant increases $\varepsilon_\mathrm{TD}^*$.
Likewise, decreasing $\kappa_\mathrm{c}$ while holding all other parameters constant also increases $\varepsilon_\mathrm{TD}^*$, because a softer leg reduces ground deformation on impact \cite{lynch2020SLP}.
Each surface has a cusp at $\phi = 1$, reflecting qualitative changes in the dynamics caused by reyielding on extension.
For $\phi \leq 1$, the minimum injected energy for which a gait exists is $\varepsilon_\mathrm{inj,min} = 2$, because the monopod penetrates to twice the weight support depth when impacting from rest ($\varepsilon_\mathrm{TD} = 0$)\footnote{In \cite{lynch2020SLP}, we show that damping in the leg can reduce the penetration depth to approximately the weight support depth and thereby reduce energy lost to ground deformation.}, whereas for $\phi > 1$, $\varepsilon_\mathrm{inj,min}$ increases with $\phi$ and decreases with $\kappa_\mathrm{c}$.
For $\varepsilon_\mathrm{inj} > \varepsilon_\mathrm{inj,min}$ and $\phi > 1$, $\varepsilon_\mathrm{TD}^*$ and efficiency $\eta$ (given in Definition \ref{def:efficiency}) decrease monotonically as $\phi$ grows.

As illustrated in Figure~\ref{fig:const-fp-surf}, parameter values corresponding to a constant $\varepsilon_\mathrm{TD}^*$ lie on a two-dimensional surface in parameter space.
This result has implications for planning, control, and terrain estimation, which we discuss in Section \ref{sec:discussion}.
This constant-$\varepsilon_\mathrm{TD}^*$ surface also has a cusp at $\phi = 1$, again reflecting the qualitative change in dynamics caused by reyielding on extension.
For $\phi \leq 1$, the energy that must be injected per hop to achieve the given fixed point decreases with stiffness ratio and is independent of force ratio.
For $\phi > 1 $, the energy that must be injected per hop to achieve a given fixed point decreases with stiffness ratio and force ratio.

\begin{figure}[t]
    \centering
    \includegraphics[width=\linewidth]{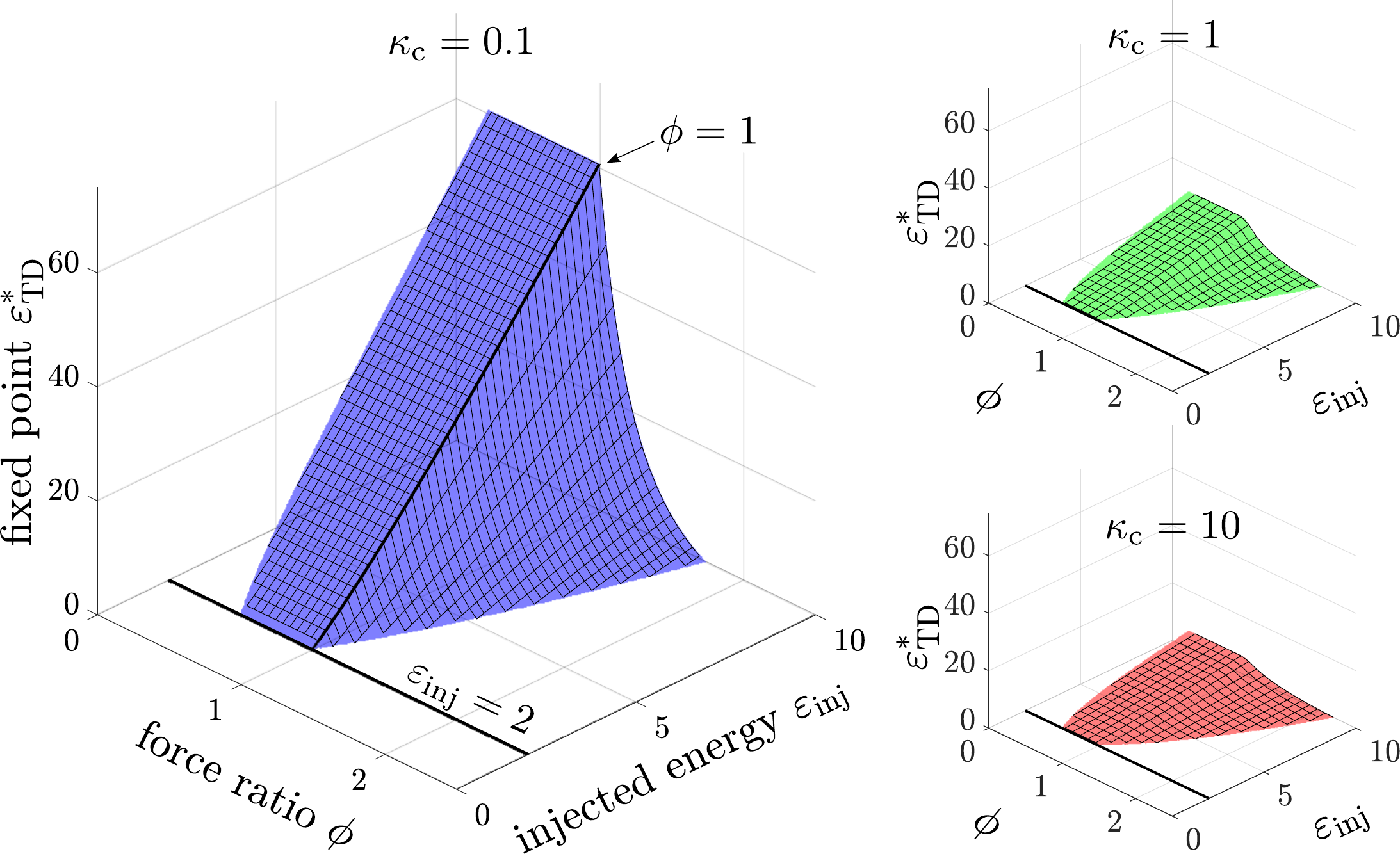}
    \caption{
    Illustration of effects of independent dimensionless parameters on fixed point of hop-to-hop energy map in the limit of negligible foot mass.
    The period-one fixed point energy $\varepsilon_\mathrm{TD}^*$ is plotted on the vertical axis \vs $\varepsilon_\mathrm{inj}$ and $\phi$ along the horizontal axes, for $\kappa_\mathrm{c} = 0.1$ (left), $\kappa_\mathrm{c} = 1$ (top right), and $\kappa_\mathrm{c} = 10$ (bottom right).
    }
    \label{fig:fxpt_surf_massless_foot_kc=0.1}
\end{figure}

\begin{figure*}[t]
    \centering
    \includegraphics[width=0.75\linewidth]{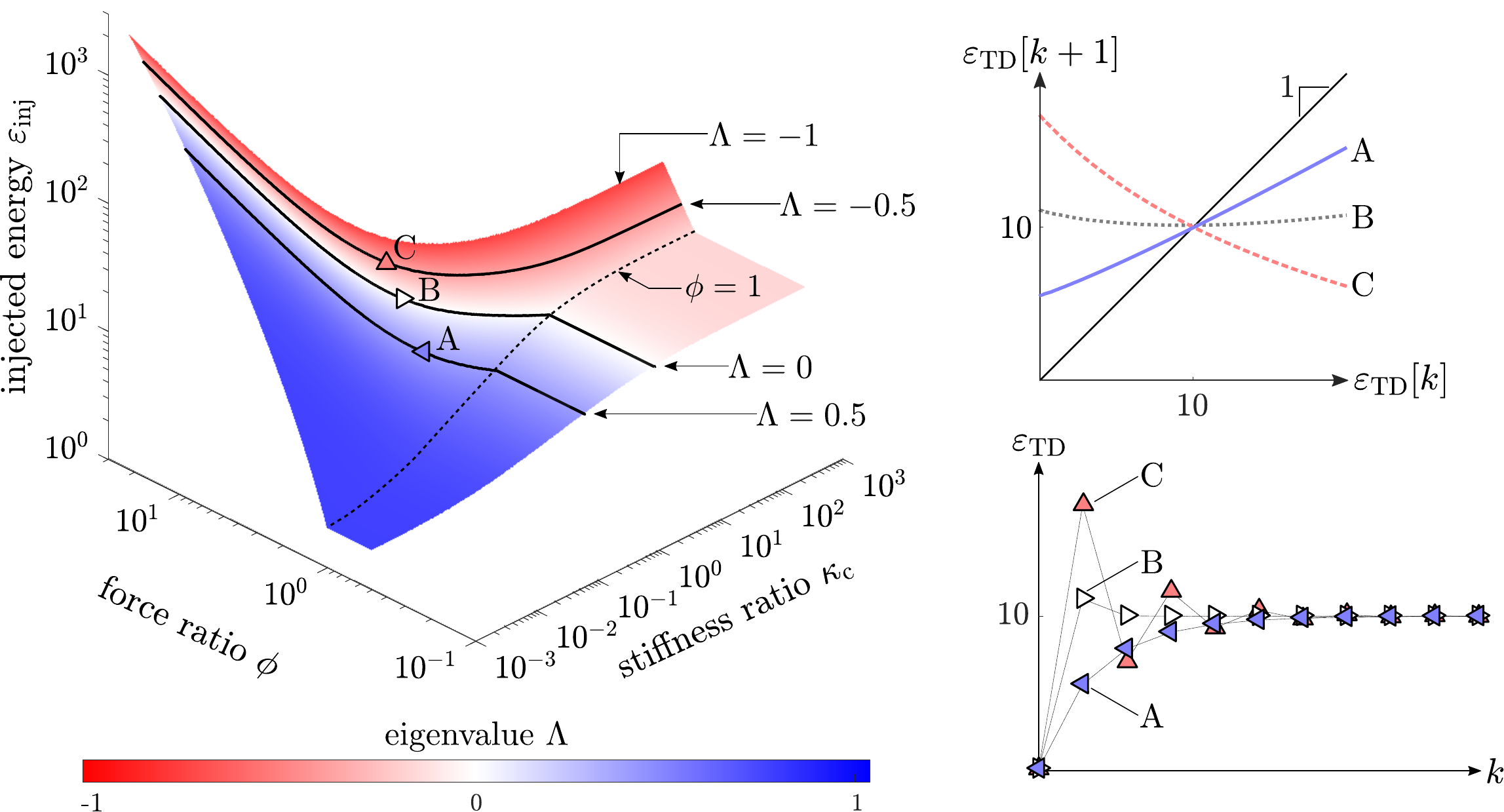}
    \caption{
    \rev{\textbf{Left:}~In the massless-foot limit, a two-dimensional surface in the three-dimensional parameter space maps to a unique period-one fixed point $\varepsilon_\mathrm{TD}^*$ ($\varepsilon_\mathrm{TD}^* = 10$ in this figure).
    The surface is only shown for parameter combinations for which the period-one fixed point is stable.
    On this surface, a one-dimensional curve of parameter values endows the fixed point with a specific eigenvalue $\Lambda$ (red-blue color map).}
    Solid black lines illustrate three constant-eigenvalue curves: $\Lambda = -0.5$, $\Lambda = 0$, and $\Lambda = 0.5$.
    \rev{We sample one point in parameter space from each of these three curves: A ($\varepsilon_\mathrm{inj} = 12.84$, $\phi = 1.52$, $\kappa_\mathrm{c} = 0.1$), B ($\varepsilon_\mathrm{inj} = 27.48$, $\phi = 2.10$, $\kappa_\mathrm{c} = 0.1$), and C ($\varepsilon_\mathrm{inj} = 45.88$, $\phi = 2.70$, $\kappa_\mathrm{c} = 0.1$).}
    \rev{\textbf{Top right:}~The hop-to-hop energy maps corresponding to parameter sets A, B, and C are overlaid atop the identity map, illustrating the same fixed point and different eigenvalues.}
    \rev{\textbf{Bottom right:}~Ten iterates of the hop-to-hop energy maps corresponding to parameter sets A, B, and C are overlaid, showing the same steady-state response and qualitatively different transient responses corresponding to their different eigenvalues.}
    }
    \label{fig:const-fp-surf}
\end{figure*}

\subsubsection{Eigenvalues}\label{sec:eigenvalues_massless}
The fact that $\varepsilon_\mathrm{loss}$ increases monotonically with $\varepsilon_\mathrm{TD}$ and the fact that $\varepsilon_\mathrm{inj}$ is constant restricts the set of possible eigenvalues for the map $\mathcal{P}_\theta$.
\begin{lemma}\label{lemma:map_slope_upper_bound}
    In the limit $\mu\to0$, the eigenvalue $\Lambda$ of the map $\mathcal{P}_\theta$ is less than one for all $\varepsilon_\mathrm{TD} \geq 0$.
\end{lemma}

\begin{proof}
    Because $\varepsilon_\mathrm{loss}$ increases monotonically with $\varepsilon_\mathrm{TD}$, $\mathrm{d}\varepsilon_\mathrm{loss}/\mathrm{d}\varepsilon_\mathrm{TD} > 0$ for all $\varepsilon_\mathrm{TD} \geq 0$.
    Therefore, from the definition of the eigenvalue given in Equation \eqref{eq:eigenvalue_def}, $\Lambda < 1$ for all $\varepsilon_\mathrm{TD} \geq 0$.
\end{proof}

Furthermore, because $\mathcal{P}_\theta$ has a closed-form representation in the limit $\mu \to 0$, the eigenvalue $\Lambda$ also has a closed-form representation in this limit.
This feature enables us to quantify the effect of different parameter values on the eigenvalue associated with a given fixed point.
In Figure \ref{fig:const-fp-surf}, we use a red-blue colormap to visualize the eigenvalue $\Lambda$ as a function of $\varepsilon_\mathrm{inj}$, $\phi$, and $\kappa_\mathrm{c}$, for $\varepsilon_\mathrm{TD}^* = 10$.
For a given fixed point, requiring a particular eigenvalue introduces a second constraint, resulting in curves in this three-dimensional parameter space along which both $\varepsilon_\mathrm{TD}^*$ and $\Lambda$ are constant.
Three such curves are plotted in Figure \ref{fig:const-fp-surf}: $\Lambda = 0.5$, which results in local monotonic convergence to $\varepsilon_\mathrm{TD}^*$, $\Lambda = 0$, which results in local deadbeat convergence, and $\Lambda = -0.5$, which results in local oscillatory convergence.
In Section~\ref{sec:discussion}, we discuss how the relationship between model parameters, fixed points, and eigenvalues can be exploited for planning and controlling gaits and for estimating terrain conditions (\ie ground stiffness).

\subsubsection{Basins of Attraction}\label{sec:BOA_massless}
We characterize conditions that guarantee global convergence to a fixed point and, when global convergence is not guaranteed, we characterize the structure of the finite \ac{BOA} associated with the fixed point.

\begin{proposition}\label{prop:global_BOA_if_minP_positive}
    In the limit $\mu\to 0$, if the fixed point $\varepsilon_\mathrm{TD}^*$ of the map $\mathcal{P}_\theta$ is linearly stable, \ie $|\Lambda| < 1$, and the minimum of the map is positive, \ie $\min \mathcal{P}_\theta > 0$, then $\varepsilon_\mathrm{TD}^*$ is globally stable, \ie its \ac{BOA} is the interval $\left[0,\infty\right)$.
\end{proposition}

\begin{proof}
While the proof is similar to Devaney's stability proof for unimodal maps on the unit interval (see Chapter 5.1 of \cite{devaney2021introduction}), the differences are substantial enough that we include the full proof \rev{in the Appendix} for completeness.
\end{proof}

\begin{remark}
    The condition $\min \mathcal{P}_\theta > 0$ distinguishes this proposed set of global stability criteria from that proposed by Koditschek and B{\"u}hler (Theorems 1 and 2 of \cite{koditschek1991hopping}), which addresses functions that map the interval $\left[0,\infty\right)$ into itself.
    On deformable terrain, if the parameters $\theta$ are such that $\min \mathcal{P}_\theta < 0$, then $\mathcal{P}_\theta$ does not map the interval $\left[0,\infty\right)$ into itself.
    \strout{We discuss this case in detail in Section \ref{sec:BOA_massless}.}
\end{remark}

\begin{proposition}\label{prop:pos_ev_implies_global_boa}
    If the map $\mathcal{P}_\theta$ has a nonnegative eigenvalue, then $\varepsilon_\mathrm{TD}^*$ is globally stable.
\end{proposition}

\noindent\textit{Sketch of proof.}
From Lemma \ref{lemma:map_slope_upper_bound}, if $\Lambda$ is nonnegative, then it must be less than one, so if the map has a nonnegative eigenvalue, $\varepsilon_\mathrm{TD}^*$ is therefore linearly stable.
Moreover, if $\mathcal{P}_\theta$ has a nonnegative eigenvalue, at most one local minimum, and no local maxima (as shown in the proof for Proposition \ref{prop:global_BOA_if_minP_positive}), then $\mathcal{P}_\theta$ cannot have a local minimum at some $\varepsilon_\mathrm{TD}$ greater than the fixed point $\varepsilon_\mathrm{TD}^*$.
Therefore, $\min \mathcal{P}_\theta > 0$, and therefore the \ac{BOA} of $\varepsilon_\mathrm{TD}^*$ is the interval $\left[0,\infty\right)$.

\begin{remark}
    This sufficient condition for global stability on soft ground is more restrictive than that proposed by Koditschek and B{\"u}hler for rigid ground, who find the \ac{BOA} of a monopod hopping on rigid ground includes all initial conditions except for a set of measure zero, provided $|\Lambda| < 1$ (Theorems 1 and 2 of \cite{koditschek1991hopping}).
\end{remark}

\begin{theorem}\label{theorem:global_boa}
    In the limit of $\mu\to0$, the \ac{BOA} of $\varepsilon_\mathrm{TD}^*$ is the entire interval $\left[0,\infty\right)$ if $\Lambda \geq 0$, or if $\min \mathcal{P}_\theta > 0~\mathrm{and}~|\Lambda| < 1$.
\end{theorem}

\begin{proof}
    The proof follows automatically from Lemma~\ref{prop:global_BOA_if_minP_positive} and Lemma~\ref{prop:pos_ev_implies_global_boa}.
\end{proof}
Figure \ref{fig:BOA_boundary_parameter_space} illustrates the region in parameter space with a global \ac{BOA} for the case $\varepsilon_\mathrm{TD}^* = 10$.

The \ac{BOA} of $\varepsilon_\mathrm{TD}^*$ changes qualitatively when $\mathcal{P}_\theta < 0$ over some interval $A_0 \subset \left[0,\infty\right)$, which corresponds to the case where the robot does not have enough energy to jump out of the hole it made in the ground.
In this case, the \ac{BOA} of $\varepsilon_\mathrm{TD}^*$ takes on a banded structure, alternating with a banded \ac{BOA} of the failed-hop attractor $A_0$.
In Figure \ref{fig:banded_BOA_example}, shaded intervals represent the \ac{BOA} for this failure mode while unshaded intervals represent the \ac{BOA} for $\varepsilon_\mathrm{TD}^*$.
An initial condition in any unshaded interval eventually converges to $\varepsilon_\mathrm{TD}^*$, whereas an initial condition in any shaded interval eventually converges to somewhere in the interval $A_0$, indicating failure to hop (in Figure \ref{fig:banded_BOA_example}, $A_0\approx\left[5,20\right]$).
As a corollary of \strout{Theorem}\rev{Proposition}~\ref{prop:pos_ev_implies_global_boa}, a necessary condition for a banded basin structure is $-1 < \Lambda < 0$, as illustrated in Figure \ref{fig:BOA_boundary_parameter_space}.
Consequently, it may be possible in practice to avoid this failure mode by designing controllers that ensure $\mathcal{P}_\theta$ has an eigenvalue between zero and positive one.
\begin{figure}[t]
    \centering
    \includegraphics[width=\linewidth]{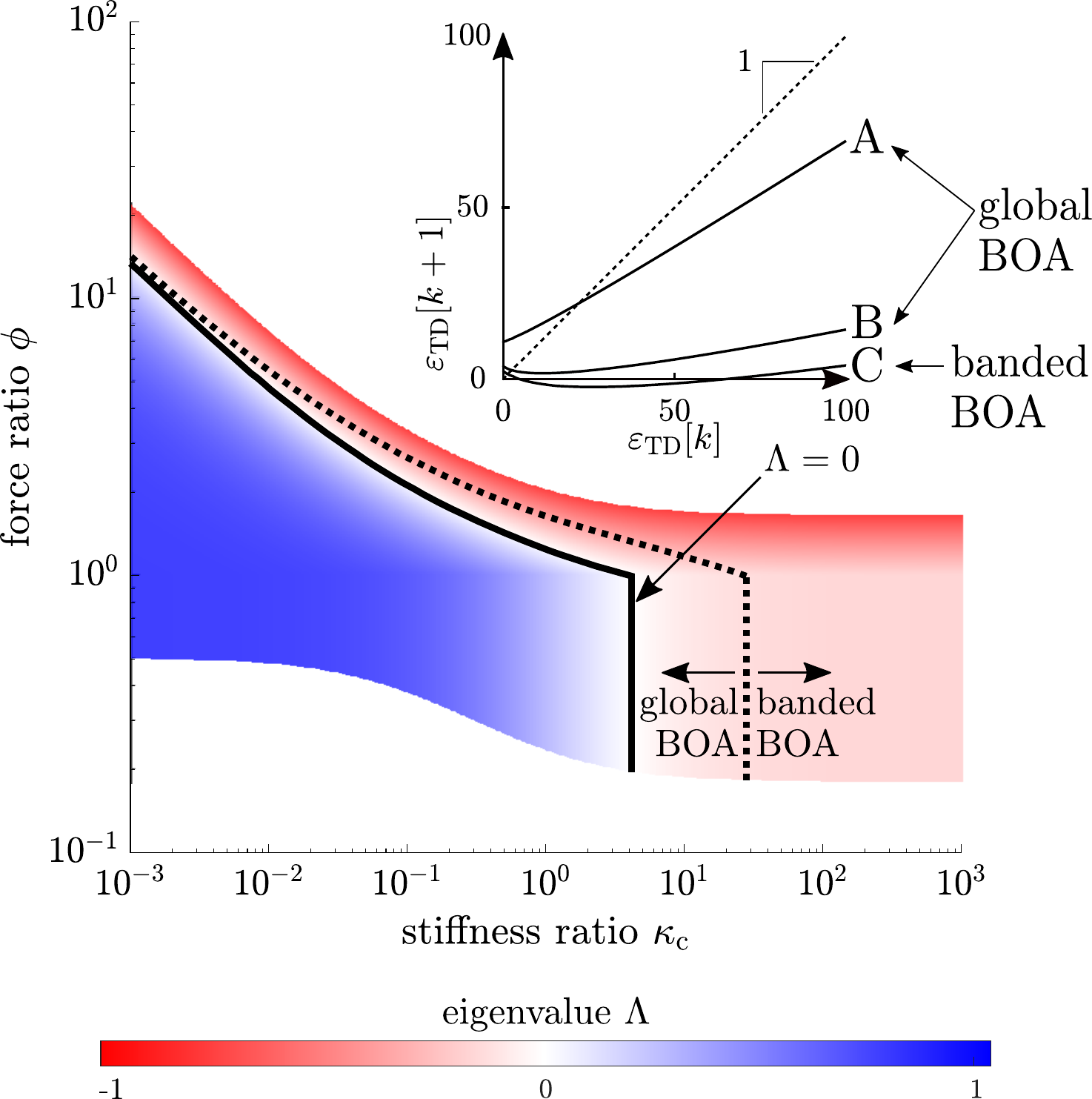}
    \caption{Dotted line indicating boundary in $\phi$---$\kappa_\mathrm{c}$ plane (for $\varepsilon_\mathrm{TD}^*=10$) that distinguishes between a global \ac{BOA} and banded \ac{BOA}.
    As $\kappa_\mathrm{c}\to 0$, this parameter-space boundary asymptotically approaches the $\Lambda = 0$ curve.
    Example maps with global \ac{BOA} and banded \ac{BOA} are sketched at top right, with $\varepsilon_\mathrm{inj,A} = 15$, $\kappa_\mathrm{c,A} = 0.2$, $\phi_\mathrm{A} = 1.25$;
    $\varepsilon_\mathrm{inj,B} = 10$, $\kappa_\mathrm{c,B} = 0.5$, $\phi_\mathrm{B} = 1.5$;
    $\varepsilon_\mathrm{inj,C} = 15$, $\kappa_\mathrm{c,C} = 0.2$, $\phi_\mathrm{C} = 2$.}
    \label{fig:BOA_boundary_parameter_space}
\end{figure}

\begin{figure}[t]
    \centering
    \includegraphics[width=\linewidth]{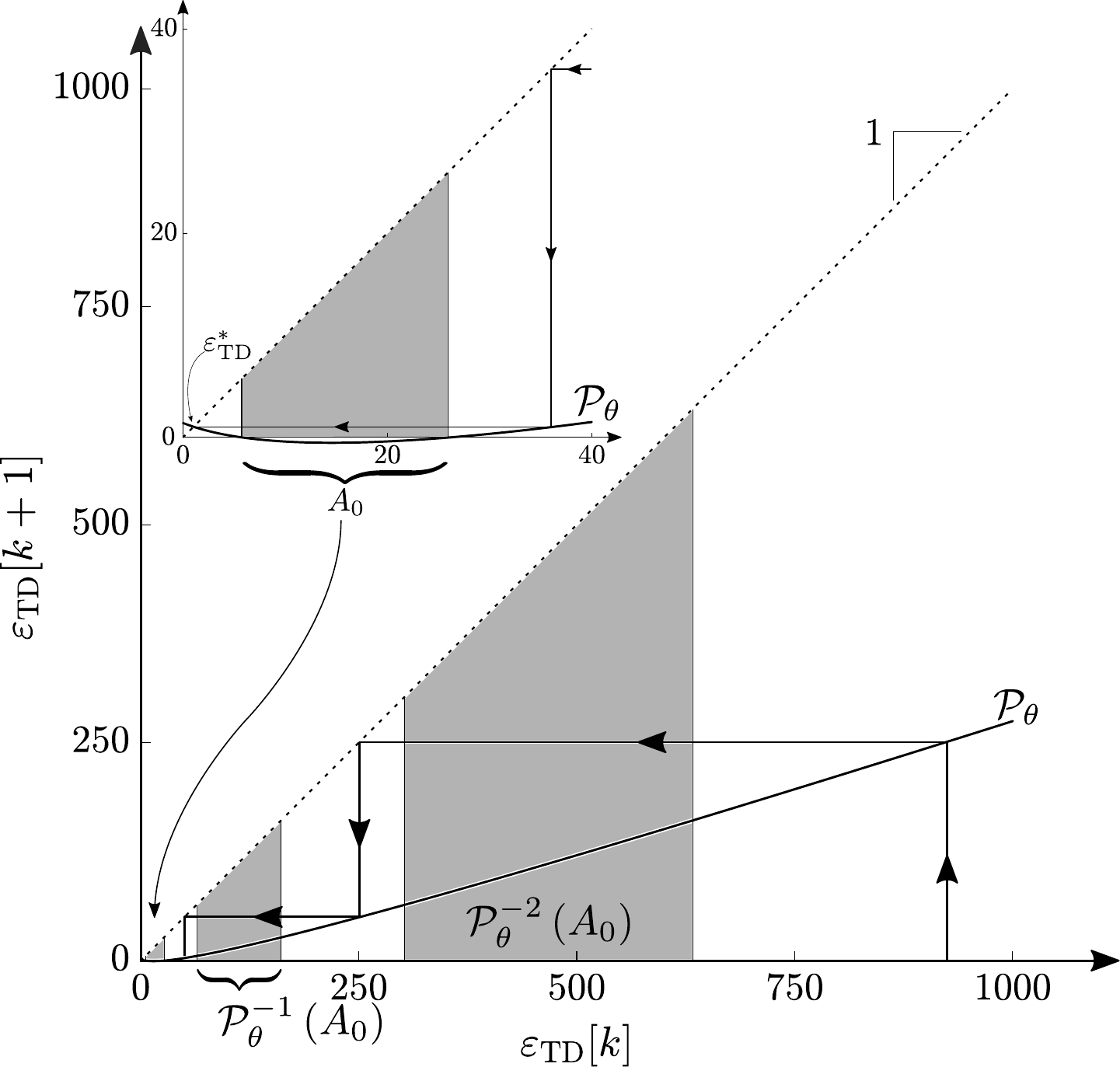}
    \caption{Example of banded basin boundaries for $\varepsilon_\mathrm{inj} = 19$, $\phi = 2.235$, and $\kappa_\mathrm{c} = 0.1$ which result in $\varepsilon_\mathrm{TD}^* = 1.04$ and $\Lambda = -0.32$.
    Inset: detailed view of $\mathcal{P}_\theta$ at low $\varepsilon_\mathrm{TD}$, including the interval $A_0$ (for which $\mathcal{P}_\theta < 0$) and its preimages.}
    \label{fig:banded_BOA_example}
\end{figure}

\section{Discussion}\label{sec:discussion}
\begin{figure*}
    \centering
    \includegraphics[width=0.9\linewidth]{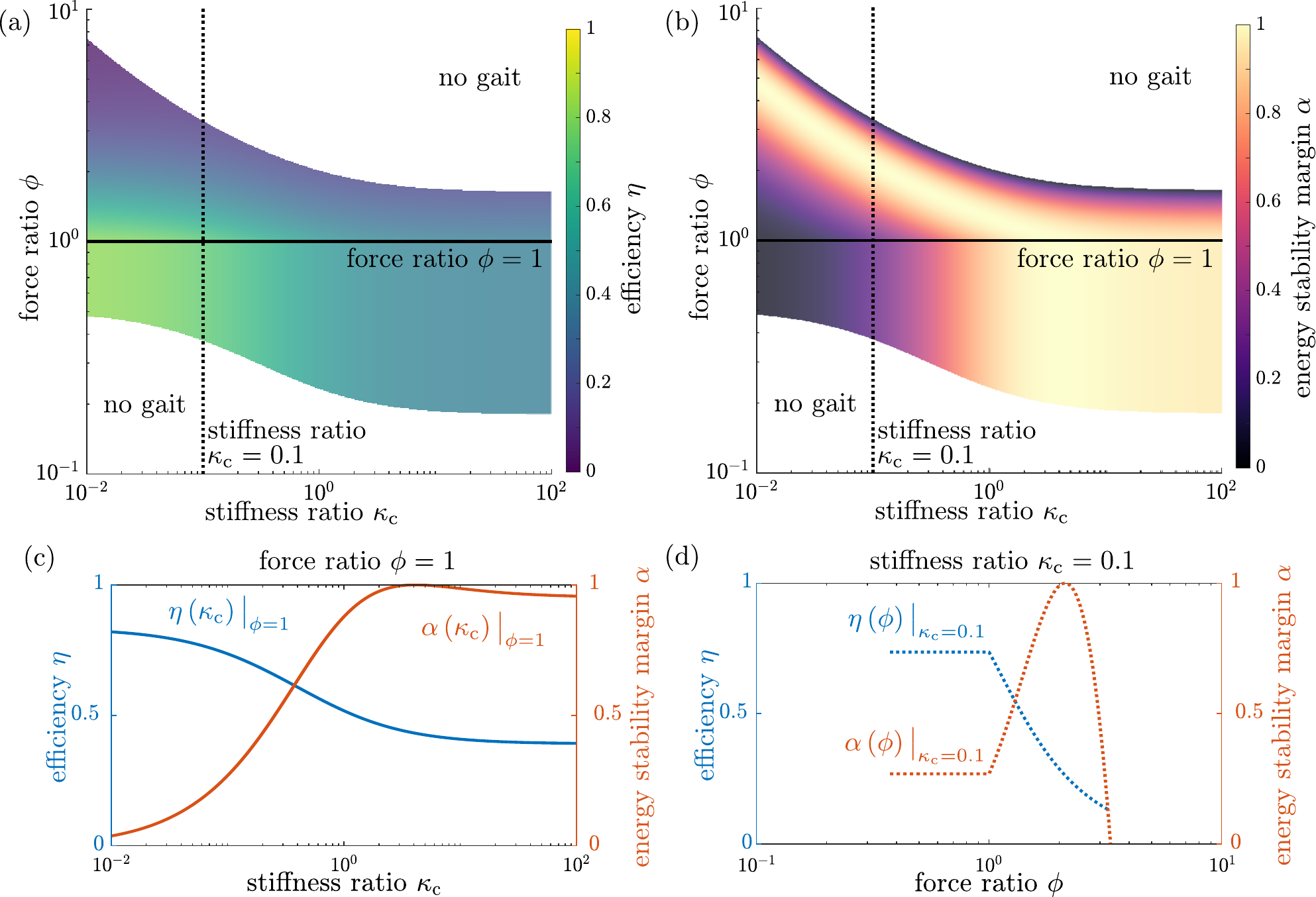}
    \caption{
    Illustration of tradeoffs between efficiency $\eta$ \strout{(see Definition \ref{def:efficiency})} and \rev{energy} stability margin $\alpha$ \strout{(see Definition \ref{def:stability_margin})} for a constant fixed point $\varepsilon_\mathrm{TD}^* = 10$\strout{, see Section~\ref{sec:discussion}}.
    \textbf{(a)}~$\eta$ \vs $\phi$ and $\kappa_\mathrm{c}$, for $\varepsilon_\mathrm{TD}^* = 10$.
    \textbf{(b)}~$\alpha$ \vs $\phi$ and $\kappa_\mathrm{c}$, for $\varepsilon_\mathrm{TD}^* = 10$.
    For both (a) and (b), white regions indicate parameter combinations for which gaits do not exist.
    \textbf{(c)}~Comparison of $\eta$ and $\alpha$ \vs $\kappa_\mathrm{c}$ for $\varepsilon_\mathrm{TD}^* = 10$ and $\phi = 1$.
    \textbf{(d)}~Comparison of $\eta$ and $\alpha$ \vs $\phi$ for $\varepsilon_\mathrm{TD}^* = 10$ and $\kappa_\mathrm{c} = 0.1$.
    }
    \label{fig:const-fp-surf_efficiency_vs_stability-margin}
\end{figure*}

The dynamics of the hop-to-hop energy map presented in this paper establish tradeoffs between efficiency, agility, and robustness, offering guidance for the design of controllers and estimators for legged locomotion on deformable terrain.
\rev{We begin this section by discussing the relevance of the vertical hopping template and the resulting hop-to-hop energy map to practical template-based controllers for legged locomotion on soft ground.
We then discuss the tradeoffs between efficiency and agility represented by the energy map, and we also discuss the structure of the map's \ac{BOA} and its implications for robustness.
Lastly, we discuss more broadly the implications these tradeoffs have for locomotion planning and control and for terrain estimation.}

\strout{Finally, apart from affecting local stability properties, the values of these three parameters also affect the \ac{BOA} of the fixed point.
Together, these relationships provide insights into planning, control, and estimation for legged locomotion on soft ground.}

\subsection{Relevance to Practical Legged Locomotion}
\rev{In template-based control, locomotion results from coupling multiple periodic low-dimensional dynamical systems (``templates'')~\cite{full1999templates,klavins2002decentralized,holmes2006dynamics}, which are then embedded in the physical robot (the ``anchor'') using methods such as those reviewed in Section~\ref{sec:introduction}.
In this paper, we have developed a vertical-hopping template applicable to running on deformable ground, analogous to Raibert's hop-height controller which assumes rigid ground.
This template and its analysis should prove valuable to roboticists using model-based control, especially template-based control, to achieve legged locomotion on deformable terrain.

On rigid ground, where the yield threshold is effectively infinite, all of the potential energy injected into the leg spring is converted to gravitational potential energy during flight, provided the foot remains stationary on the ground.
In contrast, because a deformable substrate has a finite yield threshold, some of the injected energy is dissipated into the ground rather than converted to kinetic or gravitational potential energy.
Understanding both the conditions for and the consequences of this phenomenon are key to extending template-based locomotion control to deformable terrain, where its potential benefits---interpretability, composability and generalizability, and tractability---are all the more valuable.
}

\subsection{Tradeoffs between Efficiency and Agility}
\rev{For a fixed mass ratio, three parameters govern the hop-to-hop energy dynamics: energy injected per hop $\varepsilon_\mathrm{inj}$, force ratio $\phi$, and compression-mode stiffness ratio $\kappa_\mathrm{c}$.
As illustrated in Figure \ref{fig:const-fp-surf}, parameter values that result in a particular fixed point $\varepsilon_\mathrm{TD}^*$ exist on a two-dimensional surface in this three-dimensional parameter space.
Furthermore, on this constant-fixed-point surface, a one-dimensional curve of parameter combinations endows this fixed point with a particular eigenvalue.
These features imply tradeoffs between efficiency and agility.}
As illustrated in Figure \ref{fig:const-fp-surf_efficiency_vs_stability-margin}, a tradeoff exists between efficiency $\eta$ and \rev{energy} stability margin $\alpha$\strout{, given in Definition \ref{def:efficiency} and Definition \ref{def:stability_margin}, respectively}.
Due to the finite energy injection required to sustain locomotion on deformable terrain, the maximum achievable $\eta$ is less than one. 
For a constant $\varepsilon_\mathrm{TD}^*$, $\eta$ is maximized as $\kappa_\mathrm{c}\to 0$, and $\eta$ decreases monotonically as $\kappa_\mathrm{c}$ increases, because a stiffer leg results in deeper foot penetration prior to energy injection \cite{lynch2020SLP}.
For $\phi \leq 1$, $\eta$ does not change with $\phi$, but for $\phi > 1$, $\eta$ decreases monotonically as $\phi$ increases, due to energy dissipation through reyielding upon energy injection.

Unlike $\eta$, the \rev{energy} stability margin $\alpha$ does not change monotonically with $\kappa_\mathrm{c}$ but, rather, is maximized at some finite $\kappa_\mathrm{c}$.
For $\phi \leq 1$, $\alpha$ is maximized by some $\kappa_\mathrm{c} > 1$ and does not change with $\phi$, however for $\phi > 1$, there is an inverse relationship between the $\kappa_\mathrm{c}$ and $\phi$ for which $\alpha$ is maximized.
These results suggest that reyielding on extension, while undesirable from an efficiency standpoint, may serve as a useful resource when a rapid change in \ac{COM} kinetic energy is desired, provided sufficient energy can be injected.

\subsection{Robustness}
Koditschek and B{\"u}hler use properties of S-unimodal one-dimensional maps to explain the robustness of hopping behaviors produced by Raibert's control laws on hard ground \cite{koditschek1991hopping,singer1978unimodal,guckenheimer1979unimodal}.
Our analysis of the \ac{BOA} for the hop-to-hop energy map $\mathcal{P}_\theta$ in the limit of $\mu\to 0$ adapts these results to soft ground, without reliance on S-unimodality.
Specifically, we find a wide range of dimensionless parameter values which lead to the same strong result, namely an essentially globally asymptotically stable fixed point $\varepsilon_\mathrm{TD}^*$. 
Interestingly, however, we find that this strong stability property is not guaranteed for arbitrary parameter values but, rather, only holds true when the resulting hop-to-hop energy map $\mathcal{P}_\theta$ is positive for all $\varepsilon_\mathrm{TD} \geq 0$.
Conversely, when $\mathcal{P}_\theta < 0$ over some interval, there exists a failed-hop attractor in addition to $\varepsilon_\mathrm{TD}^*$, and the basins of attraction for these two attractors exhibit an alternating banded structure as $\varepsilon_\mathrm{TD}$ increases, as shown in Figure \ref{fig:banded_BOA_example}.
In this case, it is possible to converge to the failed-hop attractor in one or more hops.
Thus, from Proposition \ref{prop:pos_ev_implies_global_boa}, a cautious heuristic for choosing control parameters might be to ensure they endow the map with an eigenvalue between zero and one, \rev{which in turn ensures $\mathcal{P}_\theta > 0$ for all $\varepsilon_\mathrm{TD} \geq 0$.}

\subsection{Implications for Planning, Control, and Terrain Estimation}
The tradeoffs between efficiency and agility discussed above suggest the use of separate locomotion modes: a ``cruise'' mode optimized for efficient locomotion on soft ground (\eg low leg stiffness relative to ground stiffness and injecting energy without reyielding) and a ``sport'' mode optimized for agile locomotion on soft ground (\eg stiffer legs, reyielding more upon energy injection, and injecting more energy per hop).
Similarly, it may be possible to adjust control parameters to enable rapid transitions between efficient gaits by optimizing parameters to maximize the \rev{energy} stability margin during transition, followed by optimizing parameters to maximize efficiency once the new gait is established.
Finally, since the stiffness ratio and the force ratio that parameterize the hop-to-hop dynamics depend on the ground stiffness,
the hop-to-hop energy map offers a discrete-time model that can be used to estimate ground stiffness during locomotion, in the spirit of Krotkov's ``every step is an experiment''~\cite{krotkov1990active}.

\section{Conclusion}
In this paper, we explore the ramifications of plastic ground deformation on hopping gaits of a spring-legged monopod driven by a Raibert-style energy injection controller.
To focus on what we consider the most important terrain feature---\strout{finite} \rev{depth-dependent} yield stress---we choose the simplest ground model with this feature: plastically-deforming ground with yield stress increasing linearly with depth.
From the resulting hybrid robot-terrain model, we derive the hop-to-hop energy map, a discrete-time one-dimensional dynamical system whose fixed points correspond to hopping gaits.
Systematic physical experiments validate the dynamics predicted by the hop-to-hop energy map for a real robot hopping on a real deformable substrate, and simulations connect our analytical results, obtained in the limit of negligible foot mass, to the reality of finite foot mass.

Analysis of the hop-to-hop energy map reveals complex boundaries in the design-control parameter space that differentiate between hopping gaits and failed hops, as well as families of transient responses and basins of attraction associated with each gait.
From this map, we propose definitions of efficiency and \rev{an energy} stability margin that serves as a proxy for agility.
We also propose global stability criteria based on the basin of attraction of the fixed point of the hop-to-hop energy map;
these stability criteria are distinct from those proposed earlier for hopping on hard ground \cite{koditschek1991hopping} and show that globally stable hopping gaits are possible on deformable terrain but depend on properties of both the robot and the terrain.
Future directions include applying this framework to \rev{template-based control of} sagittal-plane and, eventually, fully unconstrained legged locomotion on deformable terrain where the dynamics of foot-ground interaction can be substantially more complex.

\appendix
\subsection{Massless-Foot Hop-to-Hop Energy Map Derivation}
We derive the hop-to-hop energy map $\mathcal{P}_\theta$ in the limit of negligible foot mass ($\mu \to 0$).
In order to obtain agreement with low-foot-mass simulations (see Section \ref{sec:simulations} and Figure \ref{fig:sim_overlay_massless_soln}), we remove the singularity at $\mu = 0$ in the equations of motion given in Equation~\eqref{eq:nondim_EOM}.

\subsubsection{Compression Phase}
Upon \ac{TD}, the ground underfoot yields and acts like a spring in series with the leg spring; the equivalent stiffness of the two springs in series is $\kappa_\mathrm{eq,c} = \kappa_\mathrm{c}\left(1 + \kappa_\mathrm{c}\right)^{-1}$.
Likewise, the equivalent deformation of the two springs in series is $\Delta\lambda_\mathrm{eq,c} = \left(\lambda_\mathrm{c} - \xi_\mathrm{b} + \xi_\mathrm{f}\right) + \left(0 - \xi_\mathrm{f}\right) = \lambda_\mathrm{c} - \xi_\mathrm{b}$.

In the limit $\mu\to0$, the kinetic energy is zero at the compression-extension transition, and the total energy prior to the compression-extension transition, $\varepsilon_\mathrm{CE^-}$, can be rewritten in terms of an equivalent spring stiffness, $\kappa_\mathrm{eq,c}$, and an equivalent spring deformation, $\Delta\lambda_\mathrm{eq,c}$:
\begin{align}
		\varepsilon_\mathrm{CE^-} &= \frac{1}{2}\kappa_\mathrm{eq}\left(\lambda_\mathrm{c} - \xi_\mathrm{b}(\tau_\mathrm{CE})\right)^2 + \xi_\mathrm{b}(\tau_\mathrm{CE}).
\end{align}
Note that $\varepsilon_\mathrm{CE^-}$ includes recoverable energy stored in the leg spring as well as irrecoverable energy ``stored'' in the ground ``spring.''
The total energy (including energy lost to ground deformation) prior to the compression-extension transition equals the total energy at impact:
\begin{equation}
	\varepsilon_\mathrm{TD} + \lambda_\mathrm{c} = \frac{1}{2}\kappa_\mathrm{eq,c}\left(\lambda_\mathrm{c} - \xi_\mathrm{b}(\tau_\mathrm{CE})\right)^2 + \xi_\mathrm{b}(\tau_\mathrm{CE}),
\end{equation}
which can be solved for $\xi_\mathrm{b}\left(\tau_\mathrm{CE}\right)$, the body position at the compression-extension transition:
\begin{equation}
	\xi_\mathrm{b}\left(\tau_\mathrm{CE}\right) = \lambda_\mathrm{c} - \frac{1}{\kappa_\mathrm{eq,c}}\left(1 + \sqrt{1 + 2\kappa_\mathrm{eq,c}\varepsilon_\mathrm{TD}}\right).
\end{equation}
Because the leg spring and ground spring are in series when the ground is yielding, the force applied by the compression-mode leg spring, $\phi_\mathrm{a,c}$, must equal the force applied by the ground spring, $\phi_\mathrm{g}$:
\begin{equation}
	\underbrace{\kappa_\mathrm{c}\left(\lambda_\mathrm{c} - \xi_\mathrm{b} + \xi_\mathrm{f}\right)}_{\phi_\mathrm{a,c}} = \underbrace{-\xi_\mathrm{f}}_{\phi_\mathrm{g}},
\end{equation}
which can be solved for the foot position $\xi_\mathrm{f}(\tau_\mathrm{CE})$ at the compression-extension transition:
\begin{equation}\label{eq:apx:nondim_massless_foot_foot_pos_pre_CE}
	\xi_\mathrm{f}(\tau_\mathrm{CE}^-) = -\left(1 + \sqrt{1 + 2\kappa_\mathrm{eq,c}\varepsilon_\mathrm{TD}}\right).
\end{equation}
The body position at the compression-extension transition simplifies to the following:
\begin{equation}\label{eq:apx:nondim_massless_foot_body_pos_pre_CE}
	\xi_\mathrm{b}(\tau_\mathrm{CE}) = \lambda_\mathrm{c} + \frac{\xi_\mathrm{f}}{\kappa_\mathrm{eq,c}}.
\end{equation}
\begin{remark}
    As $\varepsilon_\mathrm{TD}\to 0$, $\xi_\mathrm{f} \to -2$;~\ie penetration depth for the compliant monopod is minimized when the monopod is released from rest, sinking to twice the weight-support depth.
\end{remark}

\subsubsection{Extension Phase}
The change in leg spring parameters at the compression-extension transition results in a new equilibrium configuration for the leg spring and ground spring in series, where the force in the leg spring equals the force in the ground spring:
\begin{equation}
	\kappa_\mathrm{e}\left(\lambda_\mathrm{e} - \xi_\mathrm{b} + \bar{\xi}_\mathrm{f}\right) = -\bar{\xi}_\mathrm{f},
\end{equation}
where $\xi_\mathrm{b}$, the body position, is given by Equation~\eqref{eq:apx:nondim_massless_foot_body_pos_pre_CE} and where $\bar{\xi}_\mathrm{f}$ is the equilibrium foot position.
The equilibrium foot position is then
\begin{equation}
	\bar{\xi}_\mathrm{f} = -\kappa_\mathrm{eq,e}\left(\lambda_\mathrm{e} - \xi_\mathrm{b}\right),
\end{equation}
where $\kappa_\mathrm{eq,e} = \kappa_\mathrm{e}\left(1 + \kappa_\mathrm{e}\right)^{-1}$ is the equivalent stiffness of the extension-mode leg spring and ground spring in series.
Substituting the expression for $\xi_\mathrm{b}$ given by Equation~\eqref{eq:apx:nondim_massless_foot_body_pos_pre_CE} yields the following expression for the foot position when the series combination of the extension-mode leg spring and the ground spring is in equilibrium:
\begin{equation}
	\bar{\xi}_\mathrm{f} = -\kappa_\mathrm{eq,e}\left(\lambda_\mathrm{e} - \lambda_\mathrm{c} - \frac{\xi_\mathrm{f}(\tau_\mathrm{CE}^-)}{\kappa_\mathrm{eq,c}}\right).
\end{equation}
As the leg extends and pushes down on the foot, the foot begins to move, overshooting its equilibrium position $\bar{\xi}_\mathrm{f}$ and coming to rest at depth
\begin{align}\label{eq:apx:nondim_foot_pos_post_CE}
	\begin{split}
		\xi_\mathrm{f}\left(\tau_\mathrm{CE}^+\right) &= 	\xi_\mathrm{f}\left(\tau_\mathrm{CE}^-\right) + 2\left(\bar{\xi}_\mathrm{f} - 	\xi_\mathrm{f}\left(\tau_\mathrm{CE}^-\right)\right)\\
		&= 2\bar{\xi}_\mathrm{f} - \xi_\mathrm{f}\left(\tau_\mathrm{CE}^-\right).
	\end{split}
\end{align}
Normalizing $\xi_\mathrm{f}\left(\tau_\mathrm{CE}^+\right)$ by $\xi_\mathrm{f}\left(\tau_\mathrm{CE}^-\right)$, substituting the expression for $\xi_\mathrm{f}\left(\tau_\mathrm{CE}^-\right)$ given in Equation~\eqref{eq:apx:nondim_massless_foot_foot_pos_pre_CE}, and simplifying results in the following ratio of post-reyielding foot depth to pre-reyielding foot depth:
\begin{equation}\label{eq:apx:depth_ratio_reyield}
	\frac{\xi_\mathrm{f}\left(\tau_\mathrm{CE}^+\right)}{\xi_\mathrm{f}\left(\tau_\mathrm{CE}^-\right)} = \frac{2(2\phi_\mathrm{eff} - 1)\varepsilon_\mathrm{inj}\kappa_\mathrm{c} + (\phi_\mathrm{eff}^2 \kappa_\mathrm{c} + 2\phi_\mathrm{eff} - 1)\xi_\mathrm{f}^2\left(\tau_\mathrm{CE}^-\right)}{2\varepsilon_\mathrm{inj}\kappa_\mathrm{c} + (\phi_\mathrm{eff}^2 \kappa_\mathrm{c} + 1)\xi_\mathrm{f}^2\left(\tau_\mathrm{CE}^-\right)},
\end{equation}
where $\phi_\mathrm{eff} = \max(\phi,1)$.
\begin{remark}
    The limit of this depth ratio as $\phi_\mathrm{eff}\to 1$ is one, and the limit of the depth ratio as $\kappa_\mathrm{c}\to 0$ is $2\phi_\mathrm{eff} - 1$, which agrees with the result obtained by solving the foot-ground \ac{ODE} with a constant applied force.
\end{remark}
The total energy loss per hop is $\varepsilon_\mathrm{loss} = \frac{1}{2}\xi_\mathrm{f}\left(\tau_\mathrm{CE}^+\right)^2$.
Finally, the hop-to-hop energy map in the limit of $\mu\to 0$ is expressed in closed form as
\begin{align}\label{eq:apx:nondim_massless_foot_energy_map}
\begin{split}
\mathcal{P}_\theta\left(\varepsilon_\mathrm{TD}\right) &= \varepsilon_\mathrm{TD} + \varepsilon_\mathrm{inj} - \varepsilon_\mathrm{loss}\\
&= \varepsilon_\mathrm{TD} + \varepsilon_\mathrm{inj} - \varepsilon_\mathrm{inj}\left(\frac{a \xi_\mathrm{f^-}^3 + c \xi_\mathrm{f^-}}{b \xi_\mathrm{f^-}^2 + d}\right)^2,
\end{split}
\end{align}
where $\xi_\mathrm{f^-}$ is shorthand for $\xi_\mathrm{f}(\tau_\mathrm{CE}^-)$ given in Equation \eqref{eq:apx:nondim_massless_foot_foot_pos_pre_CE}, and where
\begin{align}\label{eq:apx:abcd}
    a &= \kappa_\mathrm{c}\phi_\mathrm{eff}^2 + 2\phi_\mathrm{eff} - 1, & b &= \left(\kappa_\mathrm{c}\phi_\mathrm{eff}^2 + 1\right)\sqrt{2\varepsilon_\mathrm{inj}},\\
    c &= 2\varepsilon_\mathrm{inj}\kappa_\mathrm{c}\left(2\phi_\mathrm{eff} - 1\right), & d &= 2\varepsilon_\mathrm{inj}\kappa_\mathrm{c}\sqrt{2\varepsilon_\mathrm{inj}}.
\end{align}
Note that $a$, $b$, $c$, and $d$ are positive for $\varepsilon_\mathrm{inj} > 0$, $\phi_\mathrm{eff} \geq 1$, and $\kappa_\mathrm{c} > 0$.

\subsection{Proof of Proposition 2}
\rev{
\begin{enumerate}
    \item From Proposition \ref{prop:unique_fxpt}, the fixed point $\varepsilon_\mathrm{TD}^*$ of the map $\mathcal{P}_\theta$ is unique.
    \item From Lemma \ref{lemma:monotonic_eloss}, $\mathrm{d}\varepsilon_\mathrm{loss}/\mathrm{d}\varepsilon_\mathrm{TD} > 0$ and approaches a constant as $\varepsilon_\mathrm{TD} \to \infty$.
    Therefore, $\mathrm{d}\mathcal{P}_\theta/\mathrm{d}\varepsilon_\mathrm{TD} = 1 - \mathrm{d}\varepsilon_\mathrm{loss}/\mathrm{d}\varepsilon_\mathrm{TD}$ has at most one root, at $\varepsilon_\mathrm{TD,c}$, so $\mathcal{P}_\theta$ has at most one critical point.
    As $\varepsilon_\mathrm{TD}\to\infty$, $\mathrm{d}\mathcal{P}_\theta/\mathrm{d}\varepsilon_\mathrm{TD}$ is bounded between zero and one.
    \item Also, $\mathrm{d}^2\varepsilon_\mathrm{loss}/\mathrm{d}\varepsilon_\mathrm{TD}^2 < 0$ for all $\varepsilon_\mathrm{TD} > 0$ and vanishes only as $\varepsilon_\mathrm{TD} \to \infty$, so if $\mathcal{P}_\theta$ has a critical point, it is a local minimum.
    \item Therefore, $\mathcal{P}_\theta$ is strictly decreasing for $0 \leq \varepsilon_\mathrm{TD} < \varepsilon_\mathrm{TD,c}$ and is strictly increasing for $\varepsilon_\mathrm{TD,c} < \varepsilon_\mathrm{TD} < \infty$.
    \item Provided $\min\mathcal{P}_\theta > 0$, $\mathcal{P}_\theta^n(\varepsilon_\mathrm{TD})\to\varepsilon_\mathrm{TD}^*$ as $n\to\infty$.
    There are three cases distinguished by the magnitude of $\varepsilon_\mathrm{TD,c}$ relative to $\varepsilon_\mathrm{TD}^*$:
    \begin{itemize}
        \item $\varepsilon_\mathrm{TD,c} < \varepsilon_\mathrm{TD}^*$: $\mathcal{P}_\theta$ maps the interval $[0,\varepsilon_\mathrm{TD,c})$ into the interval $(\varepsilon_\mathrm{TD,c},\infty)$, and $\mathcal{P}_\theta^n(\varepsilon_\mathrm{TD}) \to \varepsilon_\mathrm{TD}^*$ as $n\to\infty$ for all $\varepsilon_\mathrm{TD}\in(\varepsilon_\mathrm{TD,c},\infty)$.
        \item $\varepsilon_\mathrm{TD,c} = \varepsilon_\mathrm{TD}^*$: $\mathcal{P}_\theta^n(\varepsilon_\mathrm{TD})\to\varepsilon_\mathrm{TD}^*$ as $n\to\infty$ for all $\varepsilon_\mathrm{TD}\in(\varepsilon_\mathrm{TD}^*,\infty)$.
        Also, $\mathcal{P}_\theta$ maps the interval $[0,\varepsilon_\mathrm{TD}^*)$ into the interval $(\varepsilon_\mathrm{TD}^*,\infty)$, so $\mathcal{P}_\theta^n(\varepsilon_\mathrm{TD})\to\varepsilon_\mathrm{TD}^*$ as $n\to\infty$ for all $\varepsilon_\mathrm{TD}\in[0,\varepsilon_\mathrm{TD}^*)$.
        \item $\varepsilon_\mathrm{TD,c} > \varepsilon_\mathrm{TD}^*$: Let $\hat{\varepsilon}_\mathrm{TD}$ be the preimage of $\varepsilon_\mathrm{TD}^*$ in the interval $(\varepsilon_\mathrm{TD,c},\infty)$.
        Two iterations of $\mathcal{P}_\theta$ map the interval $[\varepsilon_\mathrm{TD}^*,\hat{\varepsilon}_\mathrm{TD}]$ inside the interval $[\varepsilon_\mathrm{TD}^*,\varepsilon_\mathrm{TD,c}]$.
        It follows that $\mathcal{P}_\theta^n(\varepsilon_\mathrm{TD})\to\varepsilon_\mathrm{TD}^*$ as $n\to\infty$ for all $\varepsilon_\mathrm{TD} \in [\varepsilon_\mathrm{TD}^*,\hat{\varepsilon}_\mathrm{TD}]$.
        From graphical analysis, there exists some $k > 0$ such that $\mathcal{P}_\theta^k(\varepsilon_\mathrm{TD}) \in [\varepsilon_\mathrm{TD}^*,\hat{\varepsilon}_\mathrm{TD}]$.
        Thus $\mathcal{P}_\theta^{k+n}(\varepsilon_\mathrm{TD})\to\varepsilon_\mathrm{TD}^*$ as $n\to\infty$ for all $\varepsilon_\mathrm{TD}\in(\hat{\varepsilon}_\mathrm{TD},\infty)$.
        Finally, $\mathcal{P}_\theta$ maps the interval $[0,\varepsilon_\mathrm{TD}^*)$ into $(\varepsilon_\mathrm{TD}^*,\infty)$, so $\mathcal{P}_\theta^n(\varepsilon_\mathrm{TD})\to\varepsilon_\mathrm{TD}^*$ as $n\to\infty$ for all $\varepsilon_\mathrm{TD}\in[0,\varepsilon_\mathrm{TD}^*)$.
    \end{itemize}
\end{enumerate}
}


\section*{Acknowledgment}
We thank Igal Alterman for building the fluidized bed trackway, Blake Strebel for developing the hopping robot, and Andrew Lin for developing the lifting mechanism.

\bibliographystyle{IEEEtranN}
{\small
\bibliography{refs}}

\end{document}